%% file: main.tex
\PassOptionsToPackage{table,dvipsnames}{xcolor}
\documentclass{applemlr}

\input{apple_preamble}

\usepackage[utf8]{inputenc} %
\usepackage[T1]{fontenc}    %
\usepackage{hyperref}       %
\usepackage{url}            %
\usepackage{booktabs}       %
\usepackage{amsfonts}       %
\usepackage{nicefrac}       %
\usepackage{microtype}      %
\usepackage{xcolor}         %
\usepackage{graphicx}       %
\usepackage{subcaption}
\usepackage{amsthm,amssymb}
\usepackage{bbm}

\newtheorem{proposition}{Proposition}
\newtheorem{lemma}[proposition]{Lemma}

\newtheorem{manualtheoreminner}{Proposition}
\newenvironment{manualtheorem}[1]{%
  \ifblank{#1}
    {\renewcommand{\themanualtheoreminner}{\unskip}}
    {\renewcommand\themanualtheoreminner{#1}}%
  \manualtheoreminner
}{\endmanualtheoreminner}

\title{On the Importance of Gating: \\ Memorization vs. In-Context Learning in State Space Models}

\author[2,*]{William L. Tong}
\author[1]{Aryo Lotfi}
\author[1]{Emmanuel Abbe}
\author[1]{Kostas Vaggelakos}
\author[1]{Vishnu Banna}
\author[1]{Etai Littwin}
\author[1]{Josh Susskind}
\author[2]{Cengiz Pehlevan}
\author[1]{Eran Malach}

\affiliation[1]{Apple}
\affiliation[2]{Harvard University}
\contribution[*]{Work done while at Apple}

\abstract{
State Space Models (SSMs) have emerged as a compelling alternative to Transformers, enabling sequence modeling with constant memory and linear compute. Although SSMs exhibit reasonable performance and favorable computational characteristics, they continue to lag behind Transformers on tasks that require in-context learning and precise retrieval, slowing their adoption for large-scale language modeling. In this work, we demonstrate that both the success and failure of SSMs in these domains can be explained by studying the role of the gating mechanism, a prevalent component in modern recurrent networks. Specifically, we show through theory and experiments that this gating mechanism causes SSMs to first learn an in-weights ``memorization'' solution, while delaying, or even preventing, convergence to a correct in-context learning solution. Importantly, this happens even in cases where there are no fundamental limitations due to the architecture or its memory capacity. On the other hand, we find that gating is often beneficial for improving generalization to long sequence lengths. Our results illuminate the crucial role of the gating mechanism in shaping both the training dynamics and generalization of SSMs, and provide a basis for understanding and improving linear-time models.}

\metadata[Correspondence]{\sffamily William L. Tong: \url{wtong@g.harvard.edu}; Eran Malach: \url{e_malach@apple.com}}
\date{\sffamily\today}

\begin{document}

\maketitle

\section{Introduction}

In recent years, linear-time architectures like State Space Models (SSMs, e.g. \cite{gu2021efficiently}) and variants of Linear Attention \citep{katharopoulos2020transformers} have gained popularity for language modeling. Their primary benefit compared to Transformers is their memory and computational efficiency: their computational complexity grows linearly with the sequence length and their memory is constant, unlike Transformers that have quadratic computational complexity and linear scaling of memory\footnote{While a naive implementation of self-attention requires quadratic memory, efficient implementations such as FlashAttention \citep{dao2022flashattention} allow memory cost to be close to linear.}. Different works have demonstrated that small and medium scale linear-time architectures can reach loss that is better than Transformers \citep{gu2023mamba, yang2024gated}, and can also achieve better length generalization \citep{malach2025infinity}. However, despite years of research, linear-time architectures are still not widely used for large scale language modeling\footnote{While various recent models have adopted linear-time layers as part of a hybrid architecture \citep{qwen35blog, blakeman2025nvidia}, these models still rely on self-attention layers for enhanced retrieval capabilities.}. Indeed, several studies have identified fundamental limitations of linear-time architectures in performing in-context learning \citep{waleffe2024empirical}, copying \citep{jelassi2024repeat}, retrieval \citep{arora2023zoology} and handling long-contexts \citep{wang2025memmamba}. These observations may suggest that such architectures do not scale as well as Transformers, and that the computational efficiency comes at a significant cost in performance degradation for critical capabilities.

One common hypothesis as to why linear-time architectures lag behind Transformers is that there is simply an \emph{expressivity} gap: SSMs operate with fixed-size memory, and therefore any capability that requires large memory will be sacrificed. In this work, we explore an alternative explanation. We show that in many cases, the initialization and training dynamics of SSMs tend to favor solutions that overly rely on in-weights memorization, even when the memory capacity is sufficient for solving the task using a correct in-context learning (ICL) solution. In particular, we show that the \emph{gating} mechanism --- prevalent in many modern linear-time architectures --- biases SSMs to find memorizing solutions. At the same time, we show that the same gating mechanism enables SSMs to generalize better on sequences longer than the training data. Our main contributions are:
\begin{itemize}
    \item We theoretically analyze the training dynamics of a simplified gated linear-time model, showing that, depending on the choice of the initial gating parameter, this model converges first to a ``memorization'' solution, even when a better ICL solution can be reached. On the other hand, we show that gating enables length generalization in some cases.
    \item We validate our theoretical findings in synthetic retrieval and logical rules tasks, showing that gated SSMs may fail to find the correct ICL solution. We show that changing the gating parameter at initialization results in faster convergence in some settings.
    \item To further validate our results, we finetune Mamba models on a tool-calling task with multiple tools presented in-context. We show that by tuning the initial gating parameter (after pretraining but before SFT), we can increase or decrease the hallucination rate, showing the important role of the gating parameter in in-context learning. Interestingly, we find that we can even improve performance by using a weaker gating than learned during pretraining.
\end{itemize}

\paragraph{Related work.} Prior work finds that SSMs can learn in-context on synthetic tasks \citep{park2024can,grazzi2024mamba,li2025can}, while pretrained models remain weaker on retrieval-intensive evaluations \citep{waleffe2024empirical}. Copying and retrieval also deteriorate with distance or sequence length \citep{jelassi2024repeat,arora2023zoology}. Most closely related to our mechanism, recurrent gating has been associated both with stable length extrapolation and with recency bias, while near-no-decay ``mimetic'' initialization improves synthetic copying \citep{wang2024understanding,trockman2024mimetic}. We complement these results by clarifying the optimization competition between memorization and retrieval, and the related inference-time tradeoff between distant signal retention and distractor accumulation. Appendix \ref{app:related_work} offers a more detailed review.

\section{Theory: Retrieval with memorization ``shortcuts''} \label{sec:theory}

We begin by presenting a tractable setting to study the fundamental principles governing the relationship between gating, retrieval, and generalization. Formal statements and proofs are deferred to Appendix \ref{app:proof}. We aim to construct a retrieval task that differentiates between \textit{in-context learning} and \textit{memorization} (or ``in-weights learning"). In our setting, the model may use memorization to partially solve the task, obtaining non-trivial but imperfect accuracy; perfect performance requires in-context learning. We analyze the training dynamics of a simplified Mamba model. 

Our analysis shows that the model either converges quickly to the ICL solution (when initialized with weak gating), or enters a plateau where memorization dominates before eventually recovering in-context learning (when initialized with strong gating). We then extend our analysis to long contexts, showing that strong gating may improve generalization. We validate our theoretical results through experiments in both our synthetic setting and more complex, naturalistic tasks. Overall, our results suggest the dual role of gating: stronger gating may induce memorization over ICL, but on the other hand improves performance on long contexts in certain settings.

\subsection{Task}
Let $\rvx_1, \rvx_2, \ldots, \rvx_P \in \reals^d$ be random vectors sampled iid as $\rvx_i \sim \mathcal{N}(\vzero, \rmI / d)$. These points remain fixed through the duration of the task. To each of these $P$ points, we assign a binary label $y(\rvx_i) = \pm 1$. For each example, we sample an ordered context of length $\ell$ and a query index $q$ such that one context position contains $\rvx_q$. The context reveals (possibly perturbed) labels, while the final query token reveals only the vector $\rvx_q$. We embed the resulting input as
\[
\begin{pmatrix}
    \rvx_1 & \rvx_2 & \cdots & \rvx_\ell & \rvx_q \\
    \tilde{y}(\rvx_1) & \tilde{y}(\rvx_2) & \cdots & \tilde{y}(\rvx_\ell) & 0
\end{pmatrix} \,,
\]
where $\tilde{y}(\rvx_i)=\eta_i y(\rvx_i)$ and, independently for each context occurrence, $\Pr(\eta_i=-1)=\epsilon$ and $\Pr(\eta_i=1)=1-\epsilon$. Thus, the set of $P$ points and their underlying labels remain fixed, but the perturbations are resampled for every example. The target is the \emph{perturbed} label attached to the context occurrence of $\rvx_q$; the model needs to output $\tilde{y}(\rvx_q)$. If $\epsilon$ is small, then the model may succeed through simple \textit{memorization} by learning the mapping $y (\cdot)$. Otherwise, in order to attain high accuracy, the model must find a match for $\rvx_q$ among the $\ell$ context points to deduce the correct (potentially perturbed) label. The specific choice of input embedding, in which we stack keys and labels, follows prior theoretical work on ICL \citep{zhang2024transformer_icl,lu2025asymptotic_icl}, and lends itself to tractable, interpretable analysis. The task overall is inspired by synthetic formulations of in-context classification \citep{chan2022data,reddy2023mechanistic}.

\subsection{Model} \label{sec:simple_ssm_model}
We study a simple Mamba-like SSM based on linear attention with a global gating parameter $\lambda$. Let $\tilde{\rvx}_t = (\rvx_t , \tilde{y}(\rvx_t)) \in \reals^{d+1}$ be an augmented context token that includes the label, and let $a = -e^{\lambda}$ be the continuous time decay factor. Construct the hidden state as
$\rvh = \tilde{\rvx}_q + \sum_{t=1}^\ell e^{a (\ell - t + 1)} \, (\tilde{\rvx}_q^\intercal \tilde{\rvx}_t) \tilde{\rvx}_t$
and the model output as
$f(\tilde{\rvx}_1, \ldots, \tilde{\rvx}_\ell, \tilde{\rvx}_q) = \rvw^\intercal \rvh$, 
for some learnable readout $\rvw \in \reals^{d+1}$. We train the model through gradient descent with binary cross entropy (i.e. logistic) loss $\gL(f(\tilde{\rvx}_1, \ldots, \tilde{\rvx}_q), \tilde{y}(\rvx_q)) = \log(1 + \exp(-\tilde{y} f))$.

Our model is equivalent to a single-head S6 layer from Mamba 2, with key, query, and value matrices all set to the identity. The trainable parameters are the gating parameter $\lambda$ and readout $\rvw$. We initialize $\rvw = \vzero$, $\lambda = \lambda_0$, and study the outcomes for different choices of $\lambda_0$.

\subsection{Short context retrieval} \label{sec:short_retrieval}
We first consider retrieval in the setting where $d$ is very large and $P \ll d$. Under these conditions, the input tokens are effectively orthonormal: $\rvx_i \cdot \rvx_j \approx 0$ if $i \neq j$ and $\rvx_i \cdot \rvx_i \approx 1$. For this analysis, we assume that the tokens are exactly orthonormal and that $\epsilon \in (0,1/2)$. Without loss of generality, we focus on the learning dynamics for a single token $\rvx_i$ and fix its context position to $i$. When clear that we are referring to $\rvx_i$, we drop the arguments for $y, \tilde{y}$ and $f$, which should be read as for instance $y \equiv y(\rvx_i)$ and $f \equiv f(\tilde{\rvx}_1, \ldots, \tilde{\rvx}_\ell,\tilde{\rvx}_q)$ where $\rvx_q = \rvx_i$.

We decompose the model's readout as $\rvw = (\rvw_x,v)$, where $\rvw_x \in \reals^d$ is the weight vector associated to the input $\rvx$ and $v \in \reals$ is the weight associated with the label $\tilde{y}$. To retrieve the label, the model needs to assign mass to weight $v$. A weight configuration with $\rvw_x = 0$ and $v=1$ corresponds to a correct ICL solution. In-weights memorization corresponds to assigning non-zero weights to the input vector, thus memorizing the underlying mapping between inputs and values.  We define $u = y(\rvx_i) (\rvw_x \cdot \vx_i)$, which measures the model's prediction according to the ``memorized'' label.

We denote by $\alpha$ the decay factor between the query token and the corresponding context input, which appears at position $i$, namely $\alpha = \exp(-(\ell-i+1)e^\lambda)$. Note that strong gating (i.e. high $\lambda$) corresponds to low $\alpha$, increasing the rate at which earlier tokens are decayed in the context. Weak gating (low $\lambda$) corresponds to high $\alpha$, preserving more of the context.

We track the signed margin for point $\rvx_i$, which is given by $z = \tilde{y} f$. The signed margin tells us the degree to which a model classifies $\rvx_i$; a large, positive signed margin suggests the model classifies the point very confidently and correctly. We can show that
$z = (1 + \alpha) \eta u + \alpha v $, so the signed margin naturally decomposes into two terms (see full details in Appendix \ref{app:proof}).
We call $\delta_m = (1 + \alpha) u$ the \textit{memorization margin} and $\delta_r = \alpha v$ the \textit{retrieval margin}, so that $z = \eta \delta_m + \delta_r$. These margins quantify the degree to which the model is either memorizing or retrieving the label from context.

\begin{manualtheorem}{\ref{prop:behavior_init}}[Behavior at initialization]
    At initialization, the ratio of the initial growth rate of the retrieval margin to the initial growth rate of the memorization margin is
    $\frac{\dot{\delta}_r (0)}{\dot{\delta}_m (0)} = \frac{P \alpha^2}{(1 - 2\epsilon)(1 + \alpha)^2} $.
\end{manualtheorem}

Proposition \ref{prop:behavior_init} analyzes the behavior of the model at initialization. For large $\alpha$ (weak gating), retrieval behavior dominates, and the model is able to quickly learn to retrieve the in-context label, solving the task. However, for small $\alpha$ (strong gating), memorization dominates, and the model is insensitive to the true label. This demonstrates that strong gating dampens in-context information, favoring memorization behavior. Proposition \ref{prop:behavior_init} suggests an approximate critical scale $\alpha^* \approx \sqrt{\frac{1 - 2\epsilon}{P}}$. If $\alpha > \alpha^*$ at initialization, then the model favors retrieval; if $\alpha < \alpha^*$, the model favors memorization. We compare this prediction with trained SSMs in the experiments shown in Figure \ref{fig:theory}a.

Next, we move beyond initialization to analyze training dynamics. Even if the model begins stuck memorizing, we find that it eventually escapes the memorization solution, with escape time depending on the initial value of $\alpha$. We characterize the transition from memorization to retrieval by measuring the time at which $\delta_r > \delta_m$, i.e. when the signed margin is dominated by the retrieval margin. Our next proposition specifies the time at which this happens:

\begin{manualtheorem}{\ref{prop:escape_time}}[Informal: Escape time from memorization]
    Let $T$ be the first time at which $\delta_r(T)=\delta_m(T)$. Then, under the leading-order reduced dynamics on the memorization plateau and for small enough $\alpha_0$,
    $T = \Theta\left( \frac{1}{\alpha_0 \log (1/\alpha_0)}\right)$.
\end{manualtheorem}

A formal version of the result and its proof are given in Proposition \ref{prop:escape_time}; the assumed reduced dynamics are stated explicitly in \eqref{eq:v_reduced_dyn} and \eqref{eq:a_reduced_dyn}.
This result establishes that a model is not fatally trapped in a memorization solution, but given sufficient time, it will also learn the full retrieval solution. However, the required time may be immense. Since $\alpha = \exp(- (\ell - i + 1) e^{\lambda})$, $T$ may be exponential in the distance from the query $\ell - i + 1$, or indeed \textit{double} exponential in the scale of the initialization for $\lambda$. Hence, to learn a retrieval task for which the relevant tokens are far from the query, an initialization with a very weak gating parameter $\lambda$ (and therefore $\alpha$ close to 1) is essential to perform well in a reasonable amount of time. We validate these observations in a trained SSM in Figure \ref{fig:theory}b.

\paragraph{Interim summary}
In the short-context retrieval, the gating parameter $\lambda$ controls the competition between memorization and retrieval. When gating is strong (and $\alpha$ is small), the label-coordinate readout $v$ receives only a weak retrieval signal, while the token-specific component $\rvw_x$ can quickly learn the fixed label assignment. This creates a memorization plateau that can be escaped in time of order $\Theta((\alpha\log(1/\alpha))^{-1})$. Strong gating can delay in-context retrieval for a very long time, especially for tokens far from the query, even though the retrieval solution is eventually reachable. The short-context analysis therefore favors weak decay at initialization: preserving history makes it easier for gradient flow to discover the in-context solution before memorization dominates. 

The next section shows why this conclusion may not continue to hold for long contexts, where preserving too much context through weak gating causes noise to accumulate from too many tokens.

\subsection{Long context generalization}
\label{sec:long_retrieval}

We next consider the setting where the number of examples in the sequence $\ell$ is large. We focus on studying retrieval when the model is trained on finite $\ell$, then must generalize to arbitrary $\ell$. To isolate retrieval, we now assume that the $\rvx_i$ are random Gaussian and that the labels are drawn uniformly s.t. $ y_i \sim \{\pm 1\}$ (this corresponds to setting $P \to \infty$ and $\epsilon=1/2$). In this setting, a model trained with finite $\ell$ learns a solution where $\rvw_x = 0$ and $v = \mathcal{O}(1)$ (Proposition \ref{prop:ssm_retrieval_solution}). In the remainder of this section, we analyze the error of this solution as sequence length grows arbitrarily, identifying an ``effective context length" determined by gating outside of which retrieval fails. For this analysis, we assume that the gating parameter $\lambda$ is fixed at initialization. In practice, we find that gating remains close to its initialization in this setting, and our intuitions are representative of experiment (see Section \ref{app:gate_drift} for a more detailed discussion).

Denote $\tau = e^{-\lambda}$, which characterizes the effective context length (as we will later show). Note that strong gating corresponds to a smaller $\tau$. Let $\mathrm{Err}(\tau)$ be the error of the model given some choice of $\tau$, namely $\mathrm{Err}(\tau) = \Pr[f \ne \tilde{y}(\rvx_q)]$. The following result analyzes the error of the model at the no-decay (large $\tau$) limit:

\begin{manualtheorem}{\ref{prop:decay_for_length_gen}}[Informal: no-decay limit for long-context retrieval]
    For retrieval with $\ell\geq 2$ input vectors of dimension $d$ it holds that
    $\lim_{\tau \to \infty} \mathrm{Err}(\tau) \approx \Phi(-\sqrt{d/\ell-1})$.
\end{manualtheorem}

Proposition \ref{prop:decay_for_length_gen} (stated formally in Appendix \ref{app:proof}) shows that without decay, the relevant signal competes against all $L=\ell-1$ distractors. We validate this prediction in Figure \ref{fig:theory}c. If $L/d\to\infty$ (in particular, if $d$ is fixed and $L\to\infty$), then the no-decay signal-to-noise ratio converges to zero and the error approaches $\Phi(0)=1/2$. This is the sense in which some decay is necessary for long-context retrieval: the model must reduce the effective number of retained distractors.

The next proposition gives more information about \textit{which} positions specifically we have a chance to learn as context length increases and gating strengthens. We denote by $\mathrm{Err}(\tau, \Delta)$ the error of the model on queries with $\Delta$ distance to the context input, i.e. $\mathrm{Err}(\tau, \Delta) = \Pr[f \ne \tilde{y}(\rvx_q) | \Delta = \ell - q+1]$.

\begin{manualtheorem}{\ref{prop:local}}[Decay induces locality]
    Suppose $\ell/\tau\to\infty$, $1\leq \tau=o(d)$. For any fixed $\delta \in (0, \frac{1}{2})$, let $\Delta_{\max}$ be the maximum distance from the query such that the model's error is at most $\delta$, i.e. the maximal $\Delta$ s.t $\mathrm{Err}(\tau, \Delta) \le \delta$. Then
    $\Delta_{max} = \Theta\left( \tau \log \frac{d}{\tau} \right) \,$.
\end{manualtheorem}

Provided that $\tau \ll d$, Proposition \ref{prop:local} suggests that the distance of the furthest token away from which we can retrieve scales like $\tau\log(d/\tau)$, which we validate in Figure \ref{fig:theory}d. Thus $\tau$ controls the effective context size up to a logarithmic factor. Hence, generalization is only possible when we enable decay, but decaying induces a locality on our retrieval. Random retrieval is impossible as context length increases, but provided the relevant tokens remain close to the query, we may generalize to arbitrarily long contexts.

Together, the two regimes expose the basic tradeoff of gating. Weak gating makes it easier for gradient flow to favor retrieval over memorization in short contexts, but in long contexts it allows irrelevant tokens to accumulate enough variance to destroy the signal. Strong gating controls this variance and enables length generalization (within its effective context length $\tau$), but may push the model into a long memorization plateau for short context retrieval.

\begin{figure}
    \centering
    \includegraphics[width=0.9\linewidth]{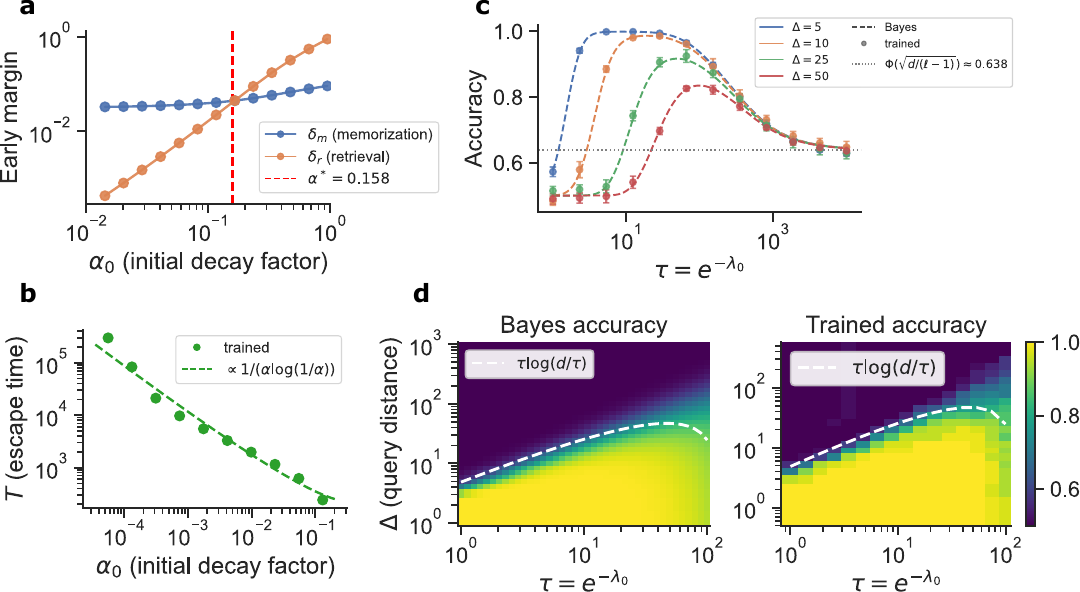}
    \caption{\textbf{Validation of theoretical predictions in a simple trained SSM.} We train the simple SSM described in Section \ref{sec:simple_ssm_model} on our synthetic retrieval task, and verify that its behavior aligns with our theoretical predictions. \textbf{(a)} The critical $\alpha^*$ predicted from Proposition \ref{prop:behavior_init} determines whether the retrieval margin $\delta_r$ or the memorization margin $\delta_m$ dominates at initialization. \textbf{(b)} The memorization plateau escape time measured from fully trained SSMs aligns closely with our predicted scaling from reduced dynamics in Proposition \ref{prop:escape_time}. \textbf{(c)} Bayes accuracy (dashed line) and experimentally measured accuracy (points) for varying $\tau$. As $\tau$ increases, the accuracy asymptotes to the prediction from Proposition \ref{prop:decay_for_length_gen}. \textbf{(d)} Heatmap comparing Bayes accuracy and accuracy of a trained SSM for varying $\tau$ and $\Delta$. The contour predicted from Proposition \ref{prop:local} fits the measured accuracies well. Appendix \ref{app:theory_exp_details} enumerates the specific configurations for each experiment.}
    \label{fig:theory}
\end{figure}

\section{Multi-Token Retrieval}
\label{sec:kv_retrieval}

In the remainder of this paper, we examine retrieval-based experiments that interrogate our theoretical intuitions in progressively more realistic settings. We test whether weaker gating promotes better in-context retrieval, and whether stronger gating promotes better length generalization. We begin by examining a simple multi-token retrieval task.

\subsection{Setup}

For a fixed $N$, each training sample presents a list of $n \sim \mathrm{Unif}[1,N]$ key-value pairs in-context, followed by a set of query keys where each key/query consists of $m$ tokens and the task is to predict the corresponding values of queried keys: $k_1:v_1,k_2:v_2, \dots,k_n:v_n|k_5:?,k_1:?,k_2:?$

More specifically, we choose a query ratio $q \in (0,1)$ and train the model in two settings: \emph{best-first}, where we query all the first $q\cdot n$ keys, and \emph{best-last}, where we query all the last $q \cdot n$ keys. The best-first setting emphasizes in-context retrieval. From Section \ref{sec:short_retrieval}, we learned that increasing the distance between the target token and query amplifies the effect of gating by reducing the decay factor $\alpha$. Hence, placing the relevant keys early in the context increases their distance to the query, amplifying the effect of gating. In contrast, the best-last setting emphasizes length generalization. From Section \ref{sec:long_retrieval}, we learned that strong gating improves length generalization but also imposes locality. We therefore predict that placing the relevant keys late in the context supports length generalization for strong gated models.

Like in our theoretical setting, we fix a static dictionary of key-value pairs shared across all examples; then for each in-context demonstration, we sample a pair from the dictionary and perturb each value token independently with probability $\epsilon$. We set $N=512$, $\epsilon=0.1$, $q=0.1$, and present each key/value using $m=8$ tokens. Under these parameters, memorizing the static mapping yields sequence accuracy $(1-\epsilon)^m \approx 0.43$; while achieving perfect accuracy requires in-context learning. 

To smoothly interpolate between standard and no-decay initialization, we use an \emph{offset} parameterization: we shift each head's initial $\lambda_0$ by a constant $\Delta \lambda_0$, i.e.,\ $\lambda_0 \leftarrow \lambda_0 + \Delta \lambda_0$. Setting $\Delta \lambda_0 = 0$ recovers the standard initialization, while $\Delta \lambda_0 \to -\infty$ approaches the no-decay regime. This allows us to study the effect of gating strength as a continuous parameter. See Appendix~\ref{app:experimental_details_kv_retrieval} for details.

\subsection{Results}
Figure \ref{fig:perturbed_kv}a and b present results for the best-first setting. When the gating parameter is small at initialization, the model learns the ICL solution. However, as the gating increases, the escape time from the memorization solution gets progressively longer. Further, Proposition \ref{prop:escape_time} predicts an approximate scaling $T \propto 1/\alpha$ up to log factors for escape time $T$ and decay factor $\alpha$. Since $\alpha \propto e^{-\text{offset}}$, we may expect that $T \propto e^{\text{offset}}$. Indeed, panel b plausibly supports a linear relationship between $T$ and $e^{\text{offset}}$.

In Figure \ref{fig:perturbed_kv}c, we consider length generalization. Using the best-last setting, we evaluate how the gating initialization affects performance when the context is extended beyond the 512 key-value pairs observed during training. We observe that stronger gating at initialization leads to better length generalization, consistent with our theoretical results. 

\begin{figure}
    \centering
    \includegraphics[width=0.9\linewidth]{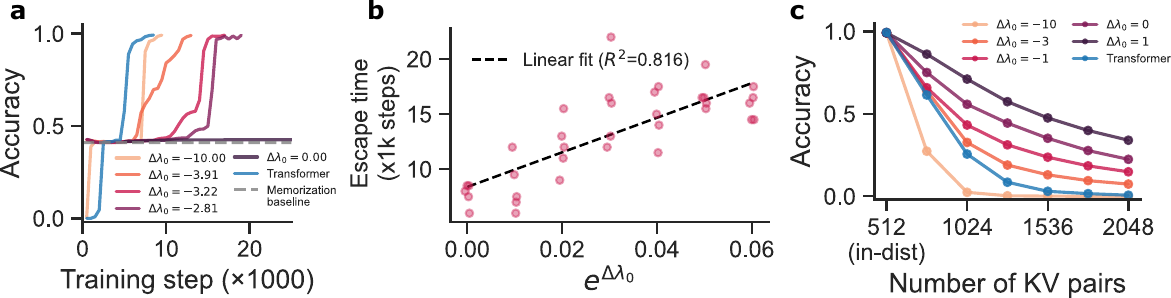}
    \caption{\textbf{Performance on multi-token retrieval.} \textbf{(a)} Accuracy across training steps for varying offset. Strong gating (large offset) never departs the memorization baseline within the training period. Smaller offset converge to perfect accuracy progressively faster. \textbf{(b)} Escape time across offsets. Proposition \ref{prop:escape_time} predicts a linear scaling, which seems plausible in this experimental setting. \textbf{(c)} Accuracy across increasing context lengths. Models were trained on up to 512 KV pairs. Strong gating (large offset) generalizes better than weaker gating.}
    \label{fig:perturbed_kv}
\end{figure}

\section{Logical Rules Retrieval} \label{sec:logic_rules_retrieval}
The tasks above are ``pure'' retrieval, where the model simply needs to match a pattern in the context. In practice, many settings involve more complex behavior, where the model needs to leverage information in the context to reason or act. To study more complex in-context retrieval, we consider agentic tool-use. 

During a tool use interaction, the model is presented with a pool of potentially relevant tools. The model must then retrieve an appropriate tool that advances the user's query. In-context retrieval is naturally essential as the model must determine from context which of many potentially novel tools to select. Length generalization is also important as the number of tools presented to the model may be variable. More tools implies a greater likelihood that a relevant one is present, but also increases the size of the context. 

This section develops a synthetic task that measures retrieval over \textit{Horn clauses}, logical statements with a convenient form for capturing knowledge and relationships. Our task is inspired by tool-calling in realistic settings while abstracting away many of the confounds. 

\subsection{Setup}
A Horn clause is a logical statement that has the form $P_1 \wedge P_2 \wedge \ldots \wedge P_n \rightarrow Q$, where the symbols $P_i, Q$ are called \textit{propositions}. 
Translated into English, this Horn clause reads ``if $P_1$ through $P_n$ are true, then $Q$" is true. Horn clauses may be combined into longer deductions. For example, suppose I have clauses (1) $P_1 \rightarrow Q$ and (2) $P_2 \wedge Q \rightarrow R$. If $P_1$ and $P_2$ are true, then I may use (1) to derive $Q$, and (2) to derive $R$, so $R$ must also be true.

Horn clauses are similar to function calls required in tool use. Like a function, a clause takes a set of inputs and produces an output. Suppose a user queries, ``what is the weather right now in California?" The query embeds a set of propositions (``the time is 3:00pm," ``the location is California"). The model must then call the ``\texttt{get\_weather}" function, which may be represented as the Horn clause (\texttt{time\_is\_3pm}), (\texttt{location\_is\_california}) $\rightarrow$ (\texttt{weather\_is\_sunny}). Horn clauses certainly do not capture all the nuances of function calls, but they offer a simple and controlled analogy that bridges our purely synthetic tasks above with the full tool use experiments we present in Section \ref{sec:tool_calling}.

This task is constructed in the following way. We define $M$ sets of propositions $\gS_1, \ldots, \gS_M$. For each pair of adjacent sets $\gS_i, \gS_{i+1}$, we sample $K$ Horn clauses with form $A \rightarrow B$ where $A \in S_i$ and $B \in S_{i+1}$. The prompt consists of a starting proposition $P \in S_1$ and a goal proposition $Q \in S_M$. The model must produce the correct sequence of Horn clauses to trace an intermediate sequence of propositions joining $P$ and $Q$. Clauses are prompted in single turns: every turn, the model produces a single clause. The starting proposition is then updated to be the output of the chosen clause. Hence, to construct a chain from $P$ to $Q$, the model outputs clauses through $M-1$ rounds of interaction. 

After being sampled, Horn clauses are fixed for the duration of the task. Each clause is assigned a random numeric name, which the model must output to select the corresponding clause. The mapping from names to clauses is presented in-context, and may vary between examples. Hence, a model must retrieve the appropriate clause in-context to advance the deduction, and cannot rely consistently on memorization of the name-to-clause association.
 During each turn, we select $k$ clauses to present in context. If the current starting proposition is in set $\gS_i$, then the input proposition of each context clause is also in $\gS_i$ \footnote{Since there are $K$ total clauses transitioning from $\gS_i$ to $\gS_{i+1}$, we require that $k < K$.}. 
Figure \ref{fig:horn}a illustrates the task.

We compare the performance of different Mamba models trained with varying gating strengths on this task. We also include a Transformer baseline. Specific details on training and task configuration are recorded in Appendix \ref{app:logic_rules_retrieval}.

\begin{figure}
    \centering
    \includegraphics[width=0.9\linewidth]{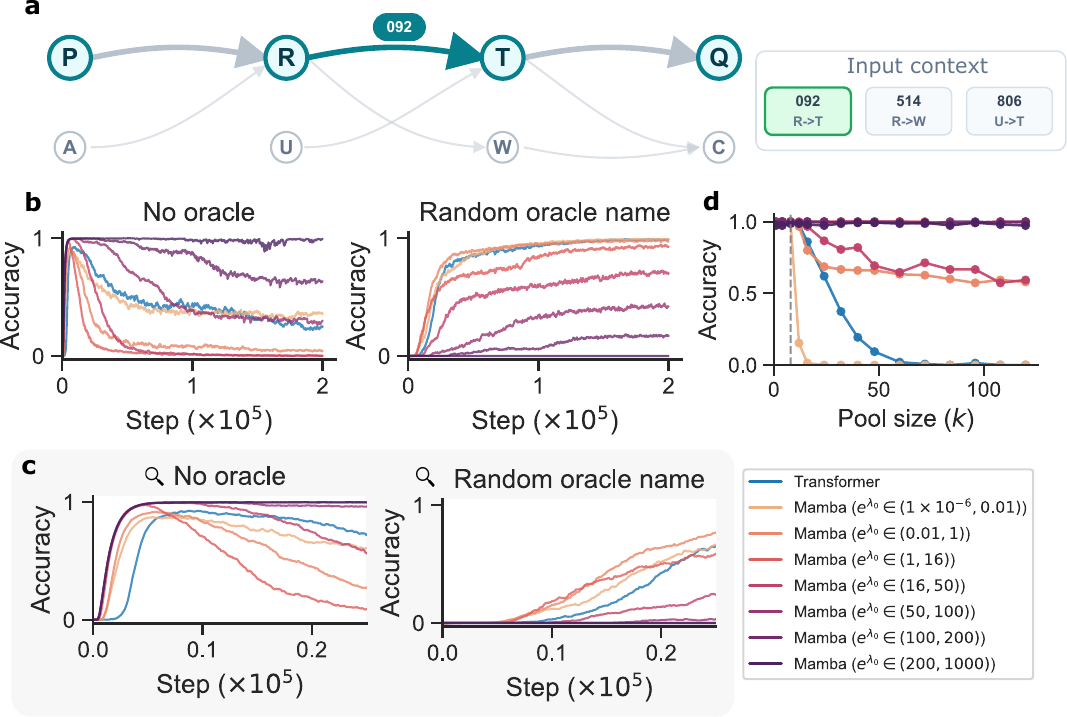}
    \caption{\textbf{Performance on logical rules retrieval.} \textbf{(a)} The logical rules retrieval task. Horn clauses transition between sets of propositions, with a starting proposition $P$ and goal proposition $Q$. The model must select the correct clause by name from its context to advance the deduction. \textbf{(b)} In-distribution retrieval performance. ``No oracle" tests memorization omitting the oracle clause in the context. ``Random oracle name" tests retrieval by including an oracle clause in the context, but resampling its name. Weaker gating implies stronger retrieval performance and faster escape from the memorization plateau. \textbf{(c)} The same as in panel (b), but zoomed in to the start of training. \textbf{(d)} Length generalization performance. Stronger gating leads to stronger generalization.} 
    \label{fig:horn}
\end{figure}

\subsection{Results}

Figure \ref{fig:horn} summarizes our results on the logic rules retrieval task. Our results support the intuition that weaker gating promotes in-context retrieval, whereas stronger gating benefits length generalization.

\paragraph{In-context retrieval.} We compare retrieval performance on in-distribution context lengths, measured by the accuracy of the first turn\footnote{It is possible that multiple paths exist from the current proposition to the goal. To account for these cases, first-turn accuracy considers a clause to be correct if the goal is reachable from the output proposition.}. Clause ordering is random, but an \textit{oracle} clause is injected at a random position in every context during training. Similar to our prior setups, the mapping from names to clauses is static. However, names among the first $q = 0.1$ proportion of clauses are resampled per example, requiring the model to retrieve these from context rather than use their memorized association. This is analogous to the ``best-first" setup from Section \ref{sec:kv_retrieval}.

Figure \ref{fig:horn}b compares accuracy in two settings, and Figure \ref{fig:horn}c zooms into early training. ``No oracle" plots accuracy when the oracle is not injected into the context. Hence, a model would attain high accuracy only if it memorizes the name-to-clause mapping. "Random oracle name" plots the accuracy on examples in which the oracle \textit{is} injected, but its name is resampled\footnote{The oracle with resampled name is injected within the first $q$ proportion of clauses, matching the training setup.}. A model attains high accuracy in this setting only if it learns to retrieve names in-context.

In general, we see that models with stronger gating tend to memorize for longer, maintaining high accuracy in the absence of an oracle but struggling to retrieve when its name is resampled. However, consistent with our theoretical intuition, the memorization phase appears to be a plateau that gives way to some degree of retrieval, where stronger decay delays the onset of retrieval behavior. For very weak gating ($e^{\lambda_0} < 1$), the model never learns a full memorization solution and never saturates accuracy in the absence of an oracle, consistent with an initial bias for learning retrieval.

\paragraph{Long context generalization.} We next consider generalization to longer contexts. Clauses are ordered according to the ranking model described in Appendix \ref{app:pl_ranking_model}, where the most relevant clause appears closest to the query, similar to the ``best-last" setup in Section \ref{sec:kv_retrieval}. All names are resampled for every example, forcing the model to learn retrieval. Models are trained on examples with variable pool size up to $k = 8$, then tested on larger pools up to the maximum $K=128$ clauses per transition.
Figure \ref{fig:horn}d plots our length generalization results. Consistent with our intuition, models with weaker gating generalize poorly, with decaying performance at longer lengths. Models with stronger gating continue generalizing perfectly to lengths well beyond the training distribution.

\section{Tool-Calling SSMs with Tool-Retrieval}
\label{sec:tool_calling}

To validate the intuitions developed in our theoretical setting and synthetic tasks, we study a naturalistic tool-calling task based on the Berkeley Function Calling Leaderboard (BFCL) \citep{patil2025bfcl}. BFCL is a leading evaluation benchmark that measures a model's ability to call tools. Given a natural language user query and pool of suggested tools, the model must execute a correct tool call. This setup furnishes a real-world tool use setting to validate our intuitions.

\subsection{Setup}
BFCL examples consist of a natural language user query
together with a tool or list of tools to call
. Tools are also presented with the arguments they take and a brief description of their function. The model must then decide on a tool, and output a well-formed function call 
, which is validated by a benchmark-specific deterministic checker. All parts of the function call must be correct in order for the example to pass. BFCL includes many different splits. We focus on evaluating the \texttt{python simple} split, where the tool calls consist of Python-style function invocations.

We extend BFCL to include variable tool pool sizes. By default, \texttt{python simple} presents only a single suggested tool. To enable variable pool sizes, we first collect all tools across all examples into a shared pool. For each example, we then use BM25 \citep{robertson2009bm25} to rank the tools according to their relevance to the user query. Finally, we take the top $k$ tools from the ranking and present them to the model. This procedure reflects standard retrieval-augmented tool-use pipelines, in which a retriever first narrows a large tool or API set before the LLM selects and invokes tools; sparse retrievers such as BM25 are common baselines in this setting \citep{patil2024gorilla,qin2023toolllm}. In cases where BM25 does not surface the correct tool for an example, we inject the correct tool at a random position in the context.

We evaluate a 2B parameter Transformer and 1.5B parameter Mamba models that were pretrained 300B tokens from the Nemotron-CC-HQ pretraining dataset \citep{su2025nemotron}, then finetuned on the Nemotron Agentic tool-calling split \citep{basant2025nvidia}. Aside from our custom tool pool system, we use the default BFCL evaluation harness to measure accuracy. See additional details on the experiment in Appendix \ref{app:bfcl_experimental_details}.

\begin{figure}
    \centering
    \includegraphics[width=0.9\linewidth]{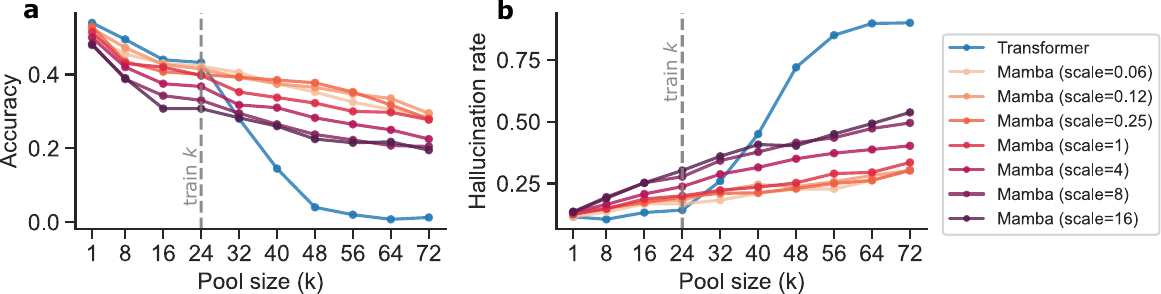}
    \caption{\textbf{Evaluation on the BFCL benchmark.} \textbf{(a)} Accuracy across tool pool sizes. Models were trained on tool pools with size up to $k=24$. \textit{Scale} is a factor $s$ applied to Mamba's decay rate before SFT (equivalently, $\lambda\leftarrow\lambda + \log s$ in our log-rate notation). Consistent with our intuition, models with weaker gating (lower scaling) perform better. Transformer performs well only on in-distribution lengths. \textbf{(b)} Hallucination rate counts examples where the model produces a syntactically valid tool call, but the tool does not exist in the input pool. Models with weaker gating also tend to hallucinate less, indicating stronger retrieval.}
    \label{fig:bfcl}
\end{figure}

\subsection{Results}

Figure \ref{fig:bfcl} summarizes our results on BFCL. All models were trained with variable pool size up to $k = 24$, then evaluated on larger pools. For the Mamba models, to control gating strength while also retaining skills learned during pretraining, we scale the gating parameter rather than re-initialize it. Larger scaling corresponds to stronger gating. Like in Section \ref{sec:kv_retrieval}, the tool pool is ordered ``best-first" so that the most relevant tool comes \textit{first}, emphasizing the need for weak gating to preserve context.

We plot both the overall accuracy of each model in Figure \ref{fig:bfcl}a, as well as the hallucination rate in Figure \ref{fig:bfcl}b. Hallucination is measured by counting examples in which the model outputs a syntactically valid tool call, but the function does not appear among the provided pool. Consistent with our intuition, we see that models with weaker gating tend to score moderately better and hallucinate less, successfully retrieving the correct tool from their context more reliably. However, gating does not appear to influence length generalization in the range we probe. Rather, all Mamba models continue to length generalize better than the Transformer baseline, suggesting some benefit from even weak gating.

\section{Discussion}
\label{sec:discussion}
In this work, we studied the role of the gating mechanism in linear-time architectures, focusing on the Mamba model. We showed that the initialization of the gating can affect training dynamics in SSMs, causing the model to prioritize memorization over in-context learning. On the other hand, we showed that gating improves length generalization in certain settings. While our results identify the gating mechanism as a crucial factor affecting in-context learning capabilities, our work does not immediately offer a solution for improving such capabilities, as our focus is primarily on scientific understanding of SSMs. We note that it is unclear whether simply changing the initialization, or even removing the gating mechanism, may be a good solution, as gating improves long-context performance in many regimes. However, it is possible that a simple gating mechanism using multiplication by a scalar is not sufficient for complex memory management required in common retrieval problems. Instead, state-transition through higher-dimensional linear operators, or even non-linear transitions, might be required for optimal retrieval in long-context settings. We leave the investigation of these directions to future work.

\section*{Acknowledgments}
WLT is supported by a Kempner Graduate Fellowship. CP is supported by an NSF CAREER Award (IIS-2239780), DARPA grants DIAL-FP-038 and AIQ-HR00112520041, the Simons Collaboration on the Physics of Learning and Neural Computation, and the William F. Milton Fund from Harvard University. This work has been made possible in part by a gift from the Chan Zuckerberg Initiative Foundation to establish the Kempner Institute for the Study of Natural and Artificial Intelligence. LLMs were used in implementing the experiments, checking proofs, and polishing the writing.

\bibliographystyle{plainnat}
\bibliography{references}

\clearpage
\appendix

\setcounter{proposition}{0}

\section{Related Work}
\label{app:related_work}
\paragraph{In-Context Learning with SSMs}
In-context learning (ICL), the ability of language models to adapt their predictions based on examples in their context, has been extensively studied in Transformer-based models \citep{dong2024survey}. More recently, different works have studied ICL in SSMs, showing that SSMs achieve competitive ICL behavior on many synthetic learning problems \citep{park2024can, grazzi2024mamba}, learn in-context with outliers \citep{li2025can} or learn from task mixtures \citep{li2024can}. While SSMs seem to perform well on ICL tasks when trained in toy settings, empirical results on large scale pretrained Mamba models indicates that they still lag behind Transformers on tasks that require retrieval and ICL abilities \citep{waleffe2024empirical}. Our work aims to bridge the gap between the positive ICL performance reported in toy settings and the limited ICL capabilities observed in real pretrained models. We show that when the data contains ``memorization'' shortcuts, where the model can memorize local patterns in-weights instead of leveraging in-context demonstrations, gated SSMs often converge to the ``memorization'' solution.

\paragraph{Copying and Retrieval with SSMs} The capabilities of SSMs to perform in-context retrieval has been studied extensively in the literature, benchmarking these models using Multi-Query Associative Recall \citep{arora2023zoology} and string copying \citep{jelassi2024repeat}. It has been shown, both empirically and theoretically, that SSMs are able to perform retrieval from relatively short sequences \citep{huang2025understanding, korenrecall}. However, retrieval and copying capabilities have been shown to deteriorate as sequence length increases \citep{jelassi2024repeat, ren2025exploring, zhan2025overcoming}. While we cannot expect SSMs to perform accurate retrieval across arbitrarily long sequences due to their fixed-size memory, it is unclear whether they perform well even in cases where memory size is not the bottleneck. For example, \cite{okpekpe2025recalling} show that SSMs are much more sensitive to choice of training hyper-parameters compared to Transformers when trained on retrieval tasks. We analyze a retrieval task, showing that Mamba models often converge first to a ``memorization'' solution, even when the architecture can perfectly solve the retrieval problem.

\paragraph{Gating and Recency Bias in SSMs} Most modern linear RNNs introduce a \emph{gating} mechanism, which decays the history state, often depending on the token input. This gating has been shown to stabilize training and improve generalization to longer sequence lengths, and is present in Gated Linear Attention \citep{yang2023gated}, Mamba \citep{gu2023mamba,dao2024transformers}, Gated DeltaNet \citep{yang2024gated}, and more. While some works have shown the benefits of the gating mechanism for outlier rejection \citep{li2025can} and in-context learning \citep{li2025gating}, other works expose some failures due to how gating is performed. For example, some works find that SSMs exhibit ``recency bias'' \citep{wang2024understanding} or local pattern shortcuts \citep{you2025revealing}, causing them to focus more on recent tokens and thus degrading performance on long context retrieval and causing information loss \cite{wang2025memmamba}. \cite{trockman2024mimetic} find that adopting ``mimetic'' initialization, which essentially disables the gating at initialization, improves learning of copying in a synthetic setting. 

\paragraph{Length Generalization in SSMs} 
The capability of SSMs to generalize to sequences longer than the ones they are trained on has been studied in different settings. \cite{gu2023mamba} demonstrate that Mamba achieves dramatically better length generalization performance compared to Transformers on an induction heads task, while other works show that the length generalization of SSMs can be improved through simple modifications \citep{ben2024decimamba, ruiz2025understanding}. More recently, \cite{malach2025infinity} show that SSMs equipped with tool-use achieve remarkable length generalization on various tasks. 
\cite{lu2025mamba} relate length generalization capabilities of Mamba to state convergence as input length increases. Our work establishes a trade-off between performance on long sequences and in-context retrieval capabilities.

\section{Detailed Theory} \label{app:proof}

In Section \ref{sec:theory} we gave a brief and informal version of our main theoretical findings. Here we give more details on the theory, including formal statements of all the results, the derivation and the proofs. 

\subsection{Short context retrieval - extended version}
Decompose $\rvw = (\rvw_x, v)$, where $\rvw_x \in \reals^d$ and $v \in \reals$. Define the following new variables:
\begin{align*}
    &\Delta = \ell - i + 1 \,, \\
    &\alpha_\Delta = \exp(-\Delta\, e^{\lambda}) \,,  \\
    &u = y(\rvx_i) (\rvw_x \cdot \rvx_i) \,,
\end{align*}
and let $\tilde{y} = \eta y$, where $\Pr(\eta = 1) = 1 - \epsilon$ and $\Pr(\eta = -1) = \epsilon$. In this fixed-position calculation we abbreviate $\alpha=\alpha_\Delta$. Then because our tokens are orthonormal, our model further simplifies to
\[f = y (\rvx_i) [(1 + \alpha) u + \alpha \eta v] \,.\]
A useful quantity to track is the signed margin for point $\rvx_i$, which is given by $z = \tilde{y} f$. The signed margin tells us the degree to which a model classifies $\rvx_i$; a large, positive signed margin suggests the model classifies the point very confidently and correctly. The signed margin simplifies to
\[z = (1 + \alpha) \eta u + \alpha v \,.\]
The signed margin naturally decomposes into two terms. We call $\delta_m = (1 + \alpha) u$ the \textit{memorization margin} and $\delta_r = \alpha v$ the \textit{retrieval margin}, so that $z = \eta \delta_m + \delta_r$.

The memorization margin represents the model's capacity to memorize the fixed label for $\rvx_i$. It derives from $\rvw_x$, the component of $\rvw$ that dots only $\rvx_i$, with no information about its context label. Its contribution to the signed margin is multiplied by $\eta$, so it is flipped when the context label mutates from the fixed assignment.

The retrieval margin represents the model's capacity to retrieve the in-context label. $v$ is the component of $\rvw$ that dots only the label coordinate, with no information about the actual token. In contrast to the memorization contribution, $\delta_r$ is not multiplied by $\eta$, so its classification is consistent regardless of whether the label mutates.

Averaging over the label noise $\eta$, the gradient flow equations are as follows. The factor $1/P$ appears only in the equation for $u$, because $u$ is specific to the point $\rvx_i$, while $v$ and $\lambda$ are shared parameters that receive symmetric contributions from all queried points.
\begin{align*}
    \dot{u} &= \frac{1}{P} (1 + \alpha) \left[ (1 - \epsilon) \sigma(-(\delta_m + \delta_r)) - \epsilon \sigma (\delta_m - \delta_r ) \right] \,, \\
    \dot{v} &= \alpha \left[ (1 - \epsilon) \sigma(-(\delta_m + \delta_r)) + \epsilon \sigma (\delta_m - \delta_r ) \right] \,, \\
    \dot{\alpha} &= (\alpha \log \alpha)^2 \left[ (1-\epsilon) (u + v) \sigma (-(\delta_m + \delta_r)) + \epsilon(v - u) \sigma (\delta_m - \delta_r) \right] \,.
\end{align*}

\begin{proposition}[Behavior at initialization] \label{prop:behavior_init}
    At initialization, the ratio of the initial growth rate of the retrieval margin to the initial growth rate of the memorization margin is
    \begin{equation} \label{eq:ssm_ratio_init}
    \frac{\dot{\delta}_r (0)}{\dot{\delta}_m (0)} = \frac{P \alpha^2}{(1 - 2\epsilon)(1 + \alpha)^2} \,.
    \end{equation}
\end{proposition}
\begin{proof}
    At initialization $u(0)=v(0)=0$, so $\delta_m(0)=\delta_r(0)=0$ and $\dot{\alpha}(0)=0$. The gradient flow equations give
    \[
    \dot{u}(0)=\frac{(1+\alpha)(1-2\epsilon)}{2P}
    \quad\text{and}\quad
    \dot{v}(0)=\frac{\alpha}{2}.
    \]
    Since $\delta_m=(1+\alpha)u$ and $\delta_r=\alpha v$, the initial margin growth rates are
    \[
    \dot{\delta}_m(0)=\frac{(1+\alpha)^2(1-2\epsilon)}{2P}
    \quad\text{and}\quad
    \dot{\delta}_r(0)=\frac{\alpha^2}{2}.
    \]
    Taking the ratio proves the claim.
\end{proof}

Proposition \ref{prop:behavior_init} suggests the behavior of the model at initialization. For large $\alpha$, retrieval behavior dominates at initialization, and the model is able to quickly learn to retrieve the in-context label, solving the task. However, for small $\alpha$, memorization dominates at initialization, and the model is insensitive to the label. This corresponds to our intuition that strong gating dampens in-context information, favoring memorization behavior. For large $P$ and small $\alpha$, \eqref{eq:ssm_ratio_init} suggests an approximate critical scale $\alpha^* \approx \sqrt{\frac{1 - 2\epsilon}{P}}$ such that if at initialization, $\alpha > \alpha^*$, the model favors retrieval, whereas if $\alpha < \alpha^*$, the model favors memorization.

If a model favors memorization initially, is it permanently trapped in such behavior, or will it uncover the retrieval solution eventually? In the remainder of this section, we establish that memorization is indeed a \textit{plateau} rather than a sink, but that escaping this plateau may take significant time.

At initialization, our gradient flow equations suggest that $\dot{u}(0) = \mathcal{O}(1 / P)$, while $\dot{v}(0) = \mathcal{O}(\alpha)$ and $\dot{\alpha}(0) = 0$. If $\alpha \ll \frac{1}{P}$, then $v$ remains effectively $0$ while $u$ grows to a finite quantity. Hence, we might reasonably model our learning dynamics by assuming that the SSM quickly attains a perfect memorization solution, then observe how $v$ and $\alpha$ evolve.

\begin{lemma}[Perfect memorization] \label{prop:perfect_mem}
    Fix $v = 0$ and $\alpha$. Then the expected loss $\E_\eta[\Ls(f,\tilde{y})]$ has a unique minimizer at
    \begin{equation*}
        u^* = \frac{\kappa}{1 + \alpha}\,,
    \end{equation*}
    where $\kappa = \log \frac{1 - \epsilon}{\epsilon} = \sigma^{-1}(1 - \epsilon)$.
\end{lemma}
\begin{proof}
    With $v=0$, the signed margin is $z=\eta(1+\alpha)u$. The expected loss as a function of $u$ is
    \[
    L(u)=(1-\epsilon)\log(1+\exp(-(1+\alpha)u))
    +\epsilon\log(1+\exp((1+\alpha)u)).
    \]
    Differentiating,
    \[
    L'(u)=(1+\alpha)\left[\sigma((1+\alpha)u)-(1-\epsilon)\right].
    \]
    Thus the only stationary point satisfies $\sigma((1+\alpha)u)=1-\epsilon$, or equivalently $(1+\alpha)u=\kappa$. Moreover,
    \[
    L''(u)=(1+\alpha)^2\sigma((1+\alpha)u)(1-\sigma((1+\alpha)u))>0,
    \]
    so this stationary point is the unique minimizer.
\end{proof}

Once the model attains the memorization manifold suggested by Proposition \ref{prop:perfect_mem}, namely $\delta_m=(1+\alpha)u=\kappa$, the corresponding time derivatives for $v$ and $\alpha$ are
\begin{align*}
    \dot{v} \big|_{\delta_m=\kappa} &= \alpha \, \phi(\delta_r) \,, \\
    \dot{\alpha} \big|_{\delta_m=\kappa} &= (\alpha \log \alpha)^2 \left[v \, \phi (\delta_r) + u^*(\alpha) \, \psi (\delta_r)\right]\,,
\end{align*}
where $u^*(\alpha)=\kappa/(1+\alpha)$ and
\begin{align*}
    \phi(\delta) &= (1-\epsilon) \sigma(-(\kappa + \delta)) + \epsilon \sigma(\kappa - \delta)\,, \\
    \psi(\delta) &= (1-\epsilon) \sigma(-(\kappa + \delta)) - \epsilon \sigma(\kappa - \delta)\,.
\end{align*}

\begin{lemma}[Reduced dynamics on the memorization plateau] \label{lem:reduced_dynamics}
    Fix $\epsilon\in(0,1/2)$. There are constants $c,C>0$, depending only on $\epsilon$, such that for all $\delta\in[0,\kappa]$,
    \[
        c \leq \phi(\delta)\leq C
        \quad\text{and}\quad
        |\psi(\delta)|\leq C|\delta|\,.
    \]
    Consequently, on the memorization manifold, whenever $0\leq\delta_r\leq\kappa$,
    \[
        \dot{\alpha}
        =
        (\alpha\log\alpha)^2\left[v\phi(\delta_r)+\mathcal{O}(\alpha |v|)\right].
    \]
\end{lemma}
\begin{proof}
    We have $\phi(0)=2\epsilon(1-\epsilon)>0$ and $\psi(0)=0$, using $\sigma(\kappa)=1-\epsilon$ and $\sigma(-\kappa)=\epsilon$. The bounds on $\phi$ follow because $\phi$ is continuous and strictly positive on the compact interval $[0,\kappa]$. The bound on $\psi$ follows from the mean-value theorem and boundedness of $\psi'$ on the same interval. Since $u^*(\alpha)\leq \kappa$ and $\delta_r=\alpha v$, the approximation for $\dot{\alpha}$ follows from the exact plateau equation above.
\end{proof}

Thus, to leading order in $\alpha$, we study the reduced dynamics
\begin{align}
    \dot{v} &= \alpha \, \phi(\delta_r) \,, \label{eq:v_reduced_dyn} \\
    \dot{\alpha} &= (\alpha \log \alpha)^2 v \, \phi (\delta_r)\,. \label{eq:a_reduced_dyn}
\end{align}
The omitted term can affect constants and the late part of the trajectory, but the early portion of the plateau dominates the asymptotic escape time below.
Note that $\dot{v} > 0$ and $\dot{\alpha} > 0$ after $v$ becomes positive. Therefore, we expect that the model will transition eventually from memorization to retrieval. One way of characterizing this transition is by measuring the time at which $\delta_r > \delta_m$. When this is true, the signed margin shifts to retrieval rather than memorization. Our next proposition specifies the time scale at which this happens under the reduced dynamics.

\begin{proposition}[Escape time from the memorization plateau] \label{prop:escape_time}
    Fix $\epsilon\in(0,1/2)$ and let $\kappa=\log((1-\epsilon)/\epsilon)$. Initialize $v(0)=0$ and $\alpha(0)=\alpha_0\in(0,1)$. Let $T$ be the first time at which $\delta_r(T)=\alpha(T)v(T)$ reaches the memorization margin $\kappa$. Then, under the dynamics given by \eqref{eq:v_reduced_dyn} and \eqref{eq:a_reduced_dyn}, as $\alpha_0\to 0$,
\begin{equation} \label{eq:ssm_escape_time}
    T = \Theta\left( \frac{1}{\alpha_0 \log (1/\alpha_0)}\right) \,.
\end{equation}
\end{proposition}
\begin{proof}
    Let $b=\log(1/\alpha)>0$ and $b_0=\log(1/\alpha_0)$, so $\alpha_0\to0$ is equivalent to $b_0\to\infty$. Before the hitting time, the retrieval margin satisfies $\delta_r=\alpha v\in[0,\kappa]$, and therefore $\dot v=\alpha\phi(\delta_r)>0$. Hence we may use $v$ as a clock. Dividing the two reduced equations gives
    \[
        \frac{d\alpha}{dv}=\alpha(\log\alpha)^2v.
    \]
    Equivalently, $db/dv=-b^2v$, and hence
    \[
        \frac{1}{b(v)}=\frac{1}{b_0}+\frac{v^2}{2},
        \qquad
        \alpha(v)=\exp\left(-\frac{1}{b_0^{-1}+v^2/2}\right).
    \]

    Let $v_T=v(T)$. We first check that the value of $v$ at the transition is $\Theta(1)$. The retrieval margin $\delta_r(v)=\alpha(v)v$ is increasing because both $\alpha(v)$ and $v$ are increasing. Since $\alpha(v)\leq 1$, reaching $\delta_r(v_T)=\kappa$ requires $v_T\geq \kappa$. Conversely, choose a constant $V$ such that $V\exp(-2/V^2)>\kappa$. From
    \[
        \frac{1}{b(V)}=\frac{1}{b_0}+\frac{V^2}{2}\geq \frac{V^2}{2},
    \]
    we get $b(V)\leq 2/V^2$, and hence $\alpha(V)=e^{-b(V)}\geq \exp(-2/V^2)$. Thus the retrieval margin at $v=V$ satisfies $\delta_r(V)=\alpha(V)V>\kappa$, so the hitting time must occur by $v=V$ and $v_T\leq V$. Therefore $v_T$ is bounded above and below by positive constants depending only on $\epsilon$, uniformly as $\alpha_0\to0$.

    On the whole interval before the transition, the retrieval margin satisfies $0\leq\delta_r(v)=\alpha(v)v\leq\kappa$. Lemma \ref{lem:reduced_dynamics} therefore lets us treat $\phi(\delta_r)$ as a constant up to multiplicative factors. Using $dt=dv/\dot v$ and $\dot v=\alpha(v)\phi(\delta_r)$,
    \[
        T
        =
        \int_0^{v_T}\frac{dv}{\alpha(v)\phi(\alpha(v)v)}
        =
        \Theta\left(\int_0^{v_T}\exp\left(\frac{b_0}{1+b_0v^2/2}\right)\,dv\right).
    \]
    
    Because $v_T$ is bounded between positive constants, it remains only to estimate this integral.

    For the lower bound, for all sufficiently large $b_0$, the interval $[0,1/b_0]$ lies inside $[0,v_T]$. On this interval, writing $x=b_0v^2/2$, we have $x\leq 1/(2b_0)$. Therefore
    \[
        b_0-\frac{b_0}{1+x}
        =
        \frac{b_0x}{1+x}
        \leq
        b_0x
        \leq
        \frac{1}{2}.
    \]
    In particular,
    \[
        \frac{b_0}{1+b_0v^2/2}\geq b_0-1.
    \]
    Hence the integral is at least $e^{b_0-1}/b_0$.

    For the upper bound, use $v_T\leq V$ and integrate over $[0,V]$. On $0\leq v\leq \sqrt{2/b_0}$, using $(1+x)^{-1}\leq 1-x/2$ for $x\in[0,1]$ gives
    \[
        \frac{b_0}{1+b_0v^2/2}
        \leq
        b_0\left(1-\frac{1}{2}\cdot\frac{b_0v^2}{2}\right)
        =
        b_0-\frac{b_0^2v^2}{4},
    \]
    so this part contributes at most
    \[
        e^{b_0}\int_0^\infty \exp(-b_0^2v^2/4)\,dv
        =
        \mathcal{O}(e^{b_0}/b_0).
    \]
  On the remaining interval $\sqrt{2/b_0}\leq v\leq V$, the denominator $1+b_0v^2/2$ is at least $2$, so the exponent is at most $b_0/2$. This tail contributes at most $V e^{b_0/2}=o(e^{b_0}/b_0)$. Thus the integral is $\Theta(e^{b_0}/b_0)$. Since $e^{b_0}=1/\alpha_0$, the claimed scaling follows.
\end{proof}

In this way, we establish that a model is not necessarily trapped in a memorization solution forever, but given sufficient time, it will also learn the full retrieval solution. However, the required time may be immense. Since $\alpha = \exp(- m e^{\lambda})$, $T$ may be exponential in the position of the token $m$, or indeed \textit{double} exponential in the scale of the base initialization for $\lambda$. Hence, to learn a retrieval task for which the relevant tokens are far from the query, an initialization with very small $\lambda$ (and therefore $\alpha$ close to 1) is essential to perform well in a reasonable amount of time.

\subsection{Long context generalization - extended version}

Let $\alpha_j = \exp(-j e^\lambda)$ denote the decay weight at distance $j$ from the query, and let $\Delta=\ell-t^*+1$ be the distance from the query to the context occurrence of $\rvx_q$. In our long-context setting, we prioritize studying retrieval, and allow the label mapping $\tilde{y}( \cdot)$ to be random (equivalent to setting $\epsilon = 0.5$). We also drop our orthogonality assumption and allow $P \rightarrow \infty$. 

Since labels are i.i.d. random, a model trained in this setting cannot memorize a stable mapping from inputs to labels, and must rely on the retrieval signal in the final coordinate of $\rvh$. We might therefore expect the readout $\rvw = (\rvw_x, v)$ to have the form $\rvw_x \approx 0$ and $v>0$. We formalize this intuition in the following proposition.

\begin{proposition}[Pure retrieval solution] \label{prop:ssm_retrieval_solution}
    Consider the pure retrieval distribution used in the long-context setting: the context length $\ell$ and dimension $d$ are finite, the inputs are iid $\rvx_i \sim \mathcal{N}(\vzero, \rmI/d)$, the context labels $\tilde{y}_i$ are iid uniform $\pm 1$ variables independent of the inputs, and the query token equals one context input $\rvx_{i^*}$ with target $\tilde{y}_{i^*}$. Let $\bar{\Ls}(\rvw_x,v,\lambda)=\E[\Ls(f,\tilde{y}_{t^*})]$ be the expected logistic loss, and run gradient flow on $\bar{\Ls}$ initialized with $\rvw_x(0)=\vzero$ and $v(0)=0$. Then for every training time $t > 0$,
    \[
        \rvw_x(t)=\vzero
        \quad\text{and}\quad
        v(t) > 0.
    \]
\end{proposition}
\begin{proof}
    Write $\rvh=(\rvh_x,g)$, where $g$ is the final coordinate. For any orthogonal matrix $\rmU\in\reals^{d\times d}$, applying $\rmU$ to every input vector maps $\rvh_x$ to $\rmU\rvh_x$, while preserving all inner products and leaving $g$ and the target $\tilde{y}_{t^*}$ unchanged. The Gaussian input distribution is rotationally invariant, so the expected loss is also invariant under $\rvw_x\mapsto \rmU\rvw_x$:
    \begin{equation} \label{eq:invariance}
        \bar{\Ls}(\rvw_x,v,\lambda)
        =
        \bar{\Ls}(\rmU\rvw_x,v,\lambda).
    \end{equation}
    Let $\rva=\nabla_{\rvw_x}\bar{\Ls}(\vzero,v,\lambda)$. Differentiating \eqref{eq:invariance} at $\rvw_x=\vzero$ gives $\rva=\rmU^\intercal\rva$ for every orthogonal $\rmU$. Thus $\rva$ is unchanged by every rotation or reflection. The only vector with this property is $\vzero$. Hence $\nabla_{\rvw_x}\bar{\Ls}(\vzero,v,\lambda)=\vzero$, so gradient flow initialized at $\rvw_x(0)=\vzero$ preserves $\rvw_x(t)=\vzero$.

    Along this trajectory, the model output is $vg$. Therefore
    \[
        \dot v
        =
        \E\left[\tilde{y}_{t^*}g\,\sigma(-\tilde{y}_{t^*}vg)\right].
    \]
    At $v=0$,
    \[
        \dot v
        =
        \frac{1}{2}\E[\tilde{y}_{t^*}g]
        =
        \frac{1}{2}\E[\alpha_\Delta\|\rvx_{t^*}\|^2]
        >0,
    \]
    because the distractor terms in $g$ have mean zero. Thus $v(t)$ moves positive immediately and cannot cross from positive to negative.
\end{proof}

In the remainder of this section, we analyze the error of a model in which $\rvw_x = 0$ and $v = 1$. As Proposition \ref{prop:ssm_retrieval_solution} suggests, this captures the classifier learned by gradient flow on the expected loss, up to an irrelevant positive rescaling of the output. We are particularly interested in the model's capacity for length generalization, so we study the case in which $\ell \rightarrow \infty$. 

Let us denote the final coordinate of $\rvh$ as $g$ (or equivalently, the output of the model when setting $\rvw_x = 0$ and $v=1$). From the simplified model's definition, we see that
\[g = \alpha_\Delta \|\rvx_q\|^2 \tilde{y}(\rvx_q) + \sum_{t \neq t^*} \alpha_{\ell - t + 1} (\rvx_q^\intercal \rvx_t) \tilde{y}(\rvx_t) \,.\]
Conditioned on a query $\rvx_q$ and label $\tilde{y}(\rvx_q)$, we observe that $g$ has the following distribution with respect to randomness in the input tokens:
\[g \,|\,\rvx_q, \tilde{y}(\rvx_q) = \mathcal{N}\left( \alpha_\Delta \|\rvx_q\|^2 \tilde{y}(\rvx_q) , \frac{1}{d} \|\rvx_q\|^2 (S_2(\ell, \lambda) - \alpha_\Delta^2 )\right) \,,\]
where $S_2(\ell, \lambda) = \sum_{j=1}^\ell \alpha_j^2$.

Intuitively, we see that the margin on $\rvx_q$ is weighted by $\alpha_\Delta$. For weaker decay, the weight $\alpha_\Delta$ is larger. At the same time, $g$ is subject to variance that increases with $S_2$. With weaker decay, $S_2$ also grows larger, drowning out the margin. Hence, success through this coordinate requires selecting a moderate decay that promotes a large margin while controlling its variance. This is in contrast to the picture in Section \ref{sec:short_retrieval}, where weaker decay was monotonically better.

To clarify this intuition, we rely on the following two lemmas. In Lemma \ref{lem:bayes_err}, we establish the Bayes error rate for the scalar statistic $g$. In Lemma \ref{lem:effective_length_scale}, we derive an asymptotic characterization of $S_2(\ell, \lambda)$, yielding an ``effective context size" as a function of our gating parameter.

\begin{lemma}[Label-coordinate Bayes error rate] \label{lem:bayes_err}
Given the scalar statistic $g$ for a query $\rvx_q$, the Bayes-optimal predictor based on $g$ is $\mathrm{sign}(g)$. The conditional error is
\[\Pr(\mathrm{sign}(g) \neq \tilde{y}(\rvx_q) \,|\, \rvx_q ) = \Phi\left( -\alpha_\Delta \|\rvx_q\| \sqrt{\frac{d}{S_2 (\ell, \lambda)-\alpha_\Delta^2}} \right) \,,\]
where $\Phi$ is the standard Gaussian CDF.
\end{lemma}
\begin{proof}
    Conditional on $\rvx_q$ and $\tilde{y}(\rvx_q)$, the signal term in $g$ is deterministic and equal to $\alpha_\Delta \|\rvx_q\|^2\tilde{y}(\rvx_q)$. For every distractor $t\neq t^*$, the inner product $\rvx_q^\intercal \rvx_t$ is Gaussian with mean zero and variance $\|\rvx_q\|^2/d$, and multiplication by the independent label $\tilde y(\rvx_t)$ leaves this Gaussian distribution unchanged. The distractor terms are independent, so their sum is Gaussian with variance $\|\rvx_q\|^2(S_2(\ell,\lambda)-\alpha_\Delta^2)/d$.

    The likelihood of $g$ under $\tilde{y}(\rvx_q)=1$ is a Gaussian with mean $\alpha_\Delta\|\rvx_q\|^2$, while the likelihood under $\tilde{y}(\rvx_q)=-1$ has the same variance and the opposite mean. With equal priors, the Bayes decision boundary for $g$ is therefore $g=0$, so the Bayes predictor based on this statistic is $\mathrm{sign}(g)$. The error probability is the probability that a Gaussian with mean $\alpha_\Delta\|\rvx_q\|^2$ and variance $\|\rvx_q\|^2(S_2(\ell,\lambda)-\alpha_\Delta^2)/d$ is negative, yielding the displayed expression.
\end{proof}

Lemma \ref{lem:bayes_err} formalizes our intuition from above, where larger $\alpha_\Delta$ (and therefore weaker decay) decreases the label-coordinate Bayes error. However, weaker decay also increases $S_2$, which in turn increases the error. Our arguments revolve around this key insight as we employ this lemma below.

\begin{lemma}[Effective context size] \label{lem:effective_length_scale}
Let $\tau = e^{-\lambda}$. If $\tau \geq 1$ and $\ell/\tau\to\infty$, then
\[S_2(\ell, \lambda) = \Theta(\tau) \,.\]
Moreover, $S_2(\ell,\lambda)-\alpha_\Delta^2=\Theta(\tau)$ uniformly over $\Delta\in\{1,\ldots,\ell\}$.
\end{lemma}
\begin{proof}
    Since $\alpha_j=\exp(-j/\tau)$,
    \[
    S_2(\ell,\lambda)=\sum_{j=1}^{\ell}e^{-2j/\tau}
    =\frac{e^{-2/\tau}(1-e^{-2\ell/\tau})}{1-e^{-2/\tau}}.
    \]
    For $\tau\geq 1$, set $x=2/\tau\in(0,2]$. The inequalities $x\leq e^x-1\leq \frac{e^2-1}{2}x$ imply
    \[
    \frac{\tau}{e^2-1}\leq \frac{1}{e^{2/\tau}-1}\leq \frac{\tau}{2}.
    \]
    Since $\ell/\tau\to\infty$ implies $1-e^{-2\ell/\tau}\to 1$, this proves $S_2(\ell,\lambda)=\Theta(\tau)$. For the second statement, the smallest value of $S_2(\ell,\lambda)-\alpha_\Delta^2$ is obtained by removing the largest term, so it is
    \[
    \sum_{j=2}^{\ell}e^{-2j/\tau}
    =\frac{e^{-4/\tau}(1-e^{-2(\ell-1)/\tau})}{1-e^{-2/\tau}}
    =\Theta(\tau).
    \]
    The matching upper bound follows from $S_2(\ell,\lambda)=\Theta(\tau)$.
\end{proof}

Lemma \ref{lem:effective_length_scale} indicates that the size of $S_2(\ell, \lambda)$ grows similarly with $\tau$. Since $S_2(\ell, \lambda)$ counts the contribution to variance from each input token in the context (weighted by the decay factor), $\tau$ can be thought of intuitively as the effective size of the context induced by a choice of gating $\lambda$. Higher $\lambda$ implies more severe decay, which in turn reduces the size of $\tau$; this is consistent with the intuition that more severe decay causes the model to forget earlier tokens, reducing the effective size of the context. We make the relationship between $\tau$ and context explicit in Proposition \ref{prop:local} below.

Our next proposition makes the no-decay limit explicit.

\begin{proposition}[No-decay limit for long-context retrieval] \label{prop:decay_for_length_gen}
    Fix $\ell\geq 2$ and $\Delta\in\{1,\ldots,\ell\}$. Let $\tau=e^{-\lambda}$ and let $p_{\tau,\Delta}(\rvx_q)$ denote the conditional label-coordinate Bayes error at distance $\Delta$. Then, as $\tau\to\infty$,

    \[
        p_{\tau,\Delta}(\rvx_q)
        \rightarrow
        \Phi\left(-\|\rvx_q\|\sqrt{\frac{d}{\ell-1}}\right).
    \]
\end{proposition}
\begin{proof}
    By Lemma \ref{lem:bayes_err},
    \[
        p_{\tau,\Delta}(\rvx_q)
        =
        \Phi\left(
            -\alpha_\Delta \|\rvx_q\|
            \sqrt{\frac{d}{S_2(\ell,\lambda)-\alpha_\Delta^2}}
        \right).
    \]
    For fixed $\ell$, $\alpha_j=e^{-j/\tau}\to 1$ for every $j\in\{1,\ldots,\ell\}$. Therefore $\alpha_\Delta\to 1$ and
    \[
        S_2(\ell,\lambda)-\alpha_\Delta^2
        =
        \sum_{j=1}^{\ell}\alpha_j^2-\alpha_\Delta^2
        \rightarrow
        \ell-1.
    \]
    The claimed limit follows by continuity of $\Phi$.
\end{proof}

Proposition \ref{prop:decay_for_length_gen} shows that without decay, the relevant signal competes against all $L=\ell-1$ distractors. If $L/d\to\infty$ (in particular, if $d$ is fixed and $L\to\infty$), then the no-decay signal-to-noise ratio converges to zero and the Bayes error approaches $\Phi(0)=1/2$. This is the sense in which some decay is necessary for long-context retrieval: the model must reduce the effective number of retained distractors. Lemma \ref{lem:effective_length_scale} says that in the regime $\ell/\tau\to\infty$, this effective number scales as $\Theta(\tau)$, so nontrivial label-coordinate retrieval requires $\tau$ not to be much larger than $d$. The next proposition gives more information about \textit{which} positions specifically we have a chance to learn as context length increases and gating strengthens.

\begin{proposition}[Decay induces locality] \label{prop:local}
    Suppose $\ell/\tau\to\infty$, $1\leq \tau=o(d)$, and $\|\rvx_q\|=\Theta(1)$. For any fixed $\delta \in (0, \frac{1}{2})$, let $\Delta_{max}$ be the maximum distance from the query such that the label-coordinate Bayes error for a token at $\Delta_{max}$ is at most $\delta$. Then
    \[\Delta_{max} = \Theta\left( \tau \log \frac{d}{\tau} \right) \,.\]
\end{proposition}
\begin{proof}
    Let $c_\delta=-\Phi^{-1}(\delta)>0$. By Lemma \ref{lem:bayes_err}, the label-coordinate Bayes error at distance $\Delta$ is at most $\delta$ if and only if
    \[
    \alpha_\Delta \|\rvx_q\|\sqrt{\frac{d}{S_2(\ell,\lambda)-\alpha_\Delta^2}}
    \geq c_\delta.
    \]
    Lemma \ref{lem:effective_length_scale} gives $S_2(\ell,\lambda)-\alpha_\Delta^2=\Theta(\tau)$ uniformly over $\Delta$. Since $\|\rvx_q\|=\Theta(1)$ and $\alpha_\Delta=e^{-\Delta/\tau}$, there are constants $c_1,c_2>0$ depending only on $\delta$ and the implicit constants in $\|\rvx_q\|=\Theta(1)$ such that every retrievable distance satisfies
    \[
    e^{-\Delta/\tau}\geq c_1\sqrt{\frac{\tau}{d}},
    \]
    and every distance satisfying $e^{-\Delta/\tau}\geq c_2\sqrt{\tau/d}$ is retrievable. Taking logarithms gives matching upper and lower bounds of the form
    \[
    \tau\left(\frac{1}{2}\log\frac{d}{\tau}-\log c_2\right)
    \leq \Delta_{max}
    \leq
    \tau\left(\frac{1}{2}\log\frac{d}{\tau}-\log c_1\right).
    \]
    Because $\tau=o(d)$, the logarithm diverges, so the additive constants are lower order. Therefore $\Delta_{max}=\Theta(\tau\log(d/\tau))$.
\end{proof}

Provided that $\tau \ll d$, Proposition \ref{prop:local} suggests that the distance of the furthest token away from which we can retrieve scales like $\tau\log(d/\tau)$. Thus $\tau$ controls the effective context size up to a logarithmic factor. Hence, generalization is only possible when we enable decay, but decaying induces a locality on our retrieval. Random retrieval is impossible as context length increases, but provided the relevant tokens remain close to the query, we may generalize to arbitrarily long contexts.

\subsection{How representative is a static gate?}
\label{app:gate_drift}

Our long-context analysis conditions on a fixed $\lambda$, whereas the gates in our full Mamba experiments remain trainable. The short-context dynamics give one reason initialization may nevertheless remain informative: the update to the gating factor contains $(\alpha\log\alpha)^2$ (see \eqref{eq:a_reduced_dyn}), which becomes small near both weak gating ($\alpha\approx 1$) and strong gating ($\alpha\approx 0$).

We also directly measure gate drift in the logical-rules length-generalization experiment of Figure \ref{fig:horn}d. For each initialization, Table \ref{tab:gate_drift} reports the median effective length $\tau=-1/\log\alpha$ across heads before and after 200k training steps, together with the corresponding absolute drift in $\alpha$. The effective lengths span five orders of magnitude at initialization. Their ordering is preserved after training, and the changes are modest relative to this sweep, although intermediate gates do move somewhat. Thus, initialization remains a useful proxy for the post-training gating regime.

\begin{table}[h]
    \caption{\textbf{Gate drift during logical-rules length-generalization training.} Medians are taken across Mamba heads in the experiment from Figure \ref{fig:horn}d.}
    \label{tab:gate_drift}
    \centering
    \small
    \begin{tabular}{lcc}
        \toprule
        Initial median $\tau$ (tokens) & Median $\tau$ after 200k steps & $\lvert\Delta\alpha\rvert$ \\
        \midrule
        $2.1\times 10^{4}$ (weakest) & $1.2\times 10^{4}$ & $4\times 10^{-5}$ \\
        $297$ & $315$ & $1\times 10^{-3}$ \\
        $17$ & $14$ & $0.024$ \\
        $2.8$ & $3.0$ & $0.083$ \\
        $1.4$ & $1.4$ & $0.067$ \\
        $0.8$ & $0.7$ & $0.053$ \\
        $0.2$ (strongest) & $0.2$ & $7\times 10^{-3}$ \\
        \bottomrule
    \end{tabular}
\end{table}

\section{Weak Tool Retrieval}
In our logical-rules retrieval (Section \ref{sec:logic_rules_retrieval}) and tool-calling (Section \ref{sec:tool_calling}) tasks, we assume that the relevant Horn clause or tool definition always appears in the context. In practice, agentic systems use cheap, often noisy tool retrievers to populate their context with relevant tools \citep{robertson2009bm25,patil2024gorilla,qin2023toolllm}. As a result, the correct tool may not always be present. 

In this appendix, we briefly highlight a setting in which we have a tool retriever whose fidelity is parameterized. A strong retriever surfaces the correct tool with high probability. A weak retriever surfaces the correct tool with low probability, but its weakness may be compensated by increasing the number of retrieved tools. In this way, we highlight another dimension in which balancing retrieval with length generalization is important: a model that generalizes to longer context lengths may accept a larger tool pool, compensating for a weaker retriever. Hence, an SSM with perfect retrieval but poor length generalization may not perform as well as an SSM that sacrifices some retrieval in favor of better length generalization.

\subsection{Setup}
To parameterize a retriever with controllable strength, we use the ranking procedure described in Section \ref{app:pl_ranking_model}, which uses a Plackett-Luce (PL) noise model with a temperature parameter $\beta$. When $\beta \rightarrow \infty$, the ranking is deterministic, such that the most relevant tool occurs first. When $\beta = 0$, ranking is purely random. Intermediate $\beta$ interpolates between these two regimes, such that items with closer ranks experience a progressively greater probability of swapping as $\beta$ decreases. We then take the top $k$ items by rank to form our input context. Hence, higher $\beta$ corresponds to stronger retrievers. We explore the performance of our models for varying $\beta$ on both the logical rules retrieval and tool calling tasks below. 

\subsection{Results}
See Figure \ref{fig:weak_retriever} for results in our weak retrieval setting. In general, we find that a Mamba model with appropriately balanced gating strength generalizes well to larger pool sizes: low gating prevents the model from generalizing effectively to large pool sizes (and thus misses the greater recall afforded at large pool sizes), while strong gating prevents the model from learning retrieval as effectively. Moderate gating allows the model to continue generalizing effectively while also learning strong retrieval, benefiting most in this setting when the tool retriever is imperfect.

\begin{figure}
    \centering
    \includegraphics[width=0.9\linewidth]{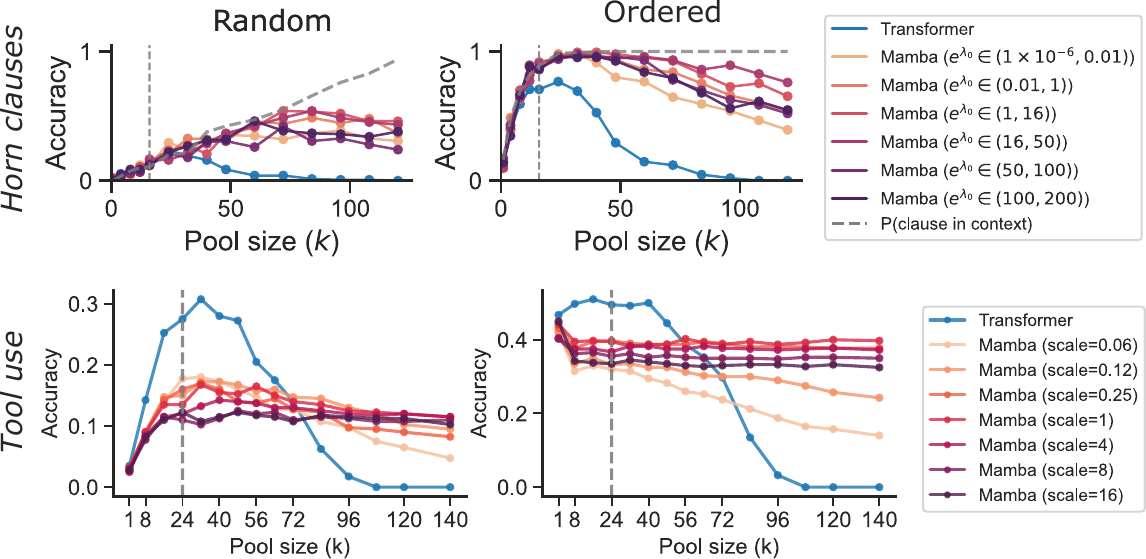}
    \caption{\textbf{Accuracy with a weak tool retriever.} We plot the accuracy of Mamba with a variety of gating strength together with a Transformer model on the logical rules retrieval task (\textit{top}) and BFCL tool use benchmark (\textit{bottom}), for both a virtually random tool retriever ($\beta = 0$ for Horn clauses, $\beta=0.05$ for tool use) and noisily ordered tool retriever ($\beta = 1$ for Horn clauses, $\beta = 10$) for tool use). The scale of $\beta$ differs between the two settings since their rankings have different scales. We see that Mamba models can generally improve or remain constant in performance with larger pool sizes, benefiting from greater recall that compensates for a weaker retriever. However, Transformers and Mamba models with low gating generalize less well, with performance falling as pool size increases.}
    \label{fig:weak_retriever}
\end{figure}

\section{Additional Experimental Details}
\label{app:experiments_details}

Below, we record the particular model and task configurations in each experiment.

\subsection{Theoretical validation experiments (Section \ref{sec:theory})} \label{app:theory_exp_details}

We record additional details for the experiments used to validate our theoretical predictions, whose results are plotted in Figure \ref{fig:theory}. We train the simplified SSM model described in Section \ref{sec:simple_ssm_model} using stochastic gradient descent with learning rate 0.1 and batch size 128; we use SGD here because it most directly matches our gradient-flow analysis. Additional panel-specific configurations are

\begin{itemize}
    \item Figure \ref{fig:theory}a (initialization bias): We configured the task with $d = 512$, $P=32$, $\ell=8$, $\Delta=4$, and $\epsilon=0.1$. The margins were computed from the model parameters after 25 gradient steps.
    \item Figure \ref{fig:theory}b (escape time): We configured the task with $d = 512$, $P=32$, $\ell=4$, $\Delta=1$, $\epsilon=0.1$, and increased the learning rate to $0.3$ to ensure escape time $T$ fell within a tractable range. 
    \item Figure \ref{fig:theory}c (Accuracy with increasing $\tau$): We configured the task with $d = 128$, $\ell=1024$, and trained for 10 thousand gradient descent steps before measuring accuracy.
    \item Figure \ref{fig:theory}d (Heatmap for varying $\tau$ and $\Delta$). We configured the task with $d=128$, $\ell=512$, and trained for 10 thousand gradient descent steps before measuring accuracy.
\end{itemize}

\subsection{Multi-token retrieval task (Section \ref{sec:kv_retrieval})}

\label{app:experimental_details_kv_retrieval}

We give more details on the perturbed key-value experiments here. In all results, both keys and values are of length $m=8$ tokens, drawn from a vocabulary of $64$ tokens. We first generate a static dictionary $D$ of key-value pairs, of size $\abs{D} = 10000$. During training, we fix a maximal number of key-value pairs $N$, and for each example we sample $n \sim \mathrm{Unif}(\{1,\dots, N\})$, sample $(k_1, v^\star_1), \dots, (k_n,v^\star_n) \sim D$ and then perturb each token in each static value $v^\star_i$ with probability $\epsilon$, to generate an example with $n$ in-context demonstrations $(k_1, v_1), \dots, (k_n,v_n)$. During evaluation, we only evaluate on the maximal choice of $N$ for our in-distribution experiments. We use $\epsilon = 0.1$ and $q = 0.1$.

\subsubsection{Model Architecture and Initialization}
\paragraph{Mamba-2.}
We use Mamba-2~\citep{dao2024transformers} with 24 layers, $d_\text{model} = 768$, expansion factor 2, 24 heads (head dimension 64), and SSM state dimension $d_\text{state} = 128$. Input and output embeddings are tied.

In Mamba-2's default initialization, each head independently draws $\lambda \sim \log(\mathrm{Uniform}(1, 16))$, giving one scalar per head with $\lambda \in [0, \log 16] \approx [0, 2.77]$. In the codebase, this parameter is referred to as $A_\text{log}$.
In the \textbf{offset} configuration, we shift all $\lambda$ values by a constant $\delta$ after this default initialization:
\[
    \lambda \leftarrow \lambda + \delta.
\]
Setting $\delta = 0$ recovers the standard initialization. Negative offsets uniformly reduce the decay, e.g.,\ $\delta = -10$ shifts the range to approximately $[-10, -7.2]$, and effectively recovers the no-decay regime. This parameterization allows us to continuously vary the gating strength while retaining the per-head variation from the default initialization.

\paragraph{Transformer.}
We use GPT-NeoX~\citep{black2022gpt} with 12 layers, $d_\text{model} = 768$, 12 attention heads, intermediate size 3072, full RoPE \citep{su2024roformer}, Flash Attention 2, no dropout, and tied embeddings.

\subsubsection{Optimization}

All models are trained with AdamW \citep{loshchilov2017decoupled} ($\beta = (0.9, 0.95)$, weight decay $0.1$, peak lr $3 \times 10^{-4}$) using linear warmup (5\% of steps) followed by cosine decay to 10\% of peak.
Gradients are clipped to norm 1.0.
Training uses bfloat16 mixed precision on 8 H100 GPUs with a global batch size of 128.
We train for up to 100k steps with early stopping at 99\% sequence accuracy.

\subsubsection{Plot descriptions for Figure~\ref{fig:perturbed_kv}}   
  All experiments in Figure~\ref{fig:perturbed_kv} use five random seeds.
  \begin{itemize}
      \item \textbf{(a) In-distribution training curves.} For each offset value, we compute the \emph{escape time} of each seed, defined as   
  the first                                                                           
  training step at which evaluation sequence accuracy exceeds 0.95. The plotted curve
  corresponds                                                                         
  to the run with the median escape time among the five seeds. The                         
  default Mamba initialization ($e^{\mathrm{offset}} = 1$) did not escape the memorization solution during our training ($100k$ steps). 
  \item  \textbf{(b) Escape time vs.\ initialization.}                                       
  For each offset, escape times of all five seeds are shown as individual points.
  The linear fit is shown                                                                   
  with $R^2$ reported in the legend.
  \item   \textbf{(c) Length generalization.}                                                 
  All models are trained on $N{=}512$ KV pairs and evaluated on                       
  $n \in \{512, 768, 1024, 1280, 1536, 1792, 2048\}$.                                 
  For each offset and evaluation length, we report the mean final sequence accuracy   
  across five seeds, with shaded bands indicating 95\% confidence intervals. 
  \end{itemize}
  
\subsection{Logic rules retrieval task (Section \ref{sec:logic_rules_retrieval})} \label{app:logic_rules_retrieval}
We provide additional details on the ranking procedure used to order clauses in context, and specific model/task configurations.

\subsubsection{Ranking clauses} \label{app:pl_ranking_model}
During a tool use interaction, agentic systems rely on tool retriever algorithms to first collect a pool of tools relevant to the user query. This procedure is often done by using an efficient ranking algorithm like BM25 to order the set of all available tools, then select the top $k$ to present to the model \citep{robertson2009bm25,patil2024gorilla,qin2023toolllm}. Tool retrievers are often designed with efficiency in mind, with a tradeoff between speed and quality. Hence, the ranking procedure may involve a degree of noise.

To model this procedure in the context of our synthetic logical rules retrieval task, we consider the following process. For a transition from $S_i$ to $S_{i+1}$, we sort the Horn clauses that implement this transition into four classes.

\begin{itemize}
    \item \textbf{Class 1}. The clause is applicable given the current set of facts \textbf{and} the goal is reachable from the resulting state
    \item \textbf{Class 2}. The clause is applicable given the current set of facts but the goal is \textbf{not} reachable from the resulting state
    \item \textbf{Class 3}. The clause is \textbf{not} applicable given the current set of facts but the goal is reachable from the resulting state
    \item \textbf{Class 4}. The clause is \textbf{not} applicable given the current set of facts and the goal is \textbf{not} reachable from the resulting state
\end{itemize}

Given these classes, we order clauses probabilistically by class using a Plackett-Luce (PL) ranking model with a temperature parameter $\beta$. When $\beta = 0$, the order is random, and when $\beta \rightarrow \infty$, the order is deterministic in class, with class 1 clauses appearing first, followed by lower classes\footnote{The position of clauses \textit{within} a class remains random}. We then take the top $k$ clauses in the order, and these become the pool of clauses presented in-context to the model. Hence, higher $\beta$ corresponds to a stronger tool retriever that more accurately orders clauses and ensures a relevant clause is likely to appear in the context.

\subsubsection{Model details}

Across all experiments, we use a Mamba-2 architecture with 4 layers, 16 heads, and the same head and state dimension of 256. We use a Transformer with a Llama-like architecture featuring 4 layers and 8 heads, each with dimension 128. All models are trained using the Muon optimizer \citep{jordan2024muon} with peak learning rate $1 \times 10^{-3}$ attained after a linear warm-up for 0.1 proportion of the total training time, followed by cosine decay. We use a batch size 128.

We use Muon because it trained both architectures substantially more reliably on this task; despite hyperparameter sweeps, AdamW did not train them consistently. Optimizer choice may affect the quantitative timing of the memorization-to-retrieval transition, although the qualitative dependence on gating is consistent across the SGD, AdamW, and Muon experiments in this paper.

\subsubsection{Task details} \label{app:horn_task_details}
The prompt has the following format:
\begin{verbatim}
    name:clause|name:clause| . . . |start_prop,goal_prop<START>
\end{verbatim}
\texttt{name} is a four character sequence of numerals that are randomly sampled for each clause. \texttt{clause} takes the form \texttt{r\_abcd->r\_efgh} where the letters correspond to alphanumeric symbols. In this way, we represent propositions as the character sequence \texttt{r\_} followed by four alphanumeric symbols. \texttt{start\_prop} is the starting/current proposition and \texttt{goal\_prop} is the goal proposition. \texttt{<START>} is a special start token appended to the end of every prompt. \texttt{|} and \texttt{,} are special separator tokens. The model outputs a sequence of four numerals, corresponding to the name of the chosen clause.

Across all tasks, we restrict Horn clauses to have a single input proposition and a single output proposition. The task consists of $M=3$ sets of propositions, necessitating two rounds of interaction. There are $128$ propositions in each set, and $K=128$ clauses transitioning between sets. 

During training, the pool size is sampled uniformly at random between 1 and some maximum $k$ (inclusive) per example. For the in-context retrieval experiments, the maximum pool size is fixed to be $k = 32$ during training. For the length generalization experiments, the maximum pool size is fixed to be $k = 8$ during training.

\subsection{Tool-calling task (Section \ref{sec:tool_calling})}
\label{app:bfcl_experimental_details}

\paragraph{Training and Optimization}
Both pretraining and SFT use AdamW. For pretraining, we train on 300B tokens, with sequence length 4K and global batch size of 256. We use cosine decay, peak learning rate of $0.0003$ and weight decay of $0.1$. The Mamba and Transformer models use the same data and training pipeline.

\paragraph{Gate scaling}
Before SFT, we multiply Mamba's decay rates $e^\lambda$ by the scale $s$ shown in Figure \ref{fig:bfcl}. In the log-rate parameterization used in Section \ref{sec:simple_ssm_model}, this intervention is exactly $\lambda \leftarrow \lambda + \log s$; it is therefore analogous to the additive offset used in the multi-token retrieval experiments.

\paragraph{Architecture}
For the SSM, we use a standard Mamba-2 architecture, with 32 layers and dimension $d = 2048$, totaling 1.5B parameters (including the embedding layer).
For the Transformer, we use a Llama-based architecture, with $28$ layers and dimension $d = 2048$, totaling 2B parameters (including the embedding layer).

\subsection{Compute Resources}
\label{app:compute}
For all experiments in the paper, we run each individual training run on one node with 8xH100 Nvidia GPUs, each experiment takes below 10 hours, with a total of around 200 experiments. For pretraining the models, we train each model (Transformer and Mamba) on 128xH100, training for less than two days. We estimate that we used under 28,288 H100 GPU hours in total for this paper.

\applefootnote{ \textcolor{textgray}{\sffamily Apple and the Apple logo are trademarks of Apple Inc., registered in the U.S. and other countries and regions.}}

\end{document}

%% file: apple_preamble.tex
\usepackage{amsmath}
\usepackage{enumerate}
\usepackage{algorithm}
\usepackage{algpseudocode}
\usepackage{amsfonts}
\usepackage{amsthm}
\usepackage{cleveref}
\usepackage{diagbox}
\usepackage{colortbl}
\usepackage{amssymb}
\usepackage{xspace}
\usepackage{wrapfig}
\usepackage{adjustbox}
\usepackage{tabularx}
\usepackage{booktabs}
\usepackage{mathtools}
\usepackage{tikz}
\usepackage{enumitem}
\usepackage{silence}
\usepackage{dsfont}
\usepackage{xcolor}
\usepackage{multirow}
\usepackage{makecell}
\usepackage{xfakebold}
\input{math_commands}

\definecolor{textgray}{HTML}{6E6E73}
\usetikzlibrary{positioning, calc}
\usetikzlibrary{decorations.pathmorphing}

\makeatletter
\patchcmd{\wrong@fontshape}{\@gobbletwo}{}{}{}
\makeatother
\numberwithin{equation}{section}
\makeatletter
\AtBeginDocument{
  \urlstyle{sf}
  
}
\makeatother

\definecolor{light}{RGB}{125, 125, 125}
\crefname{tcb@cnt@pbox}{code}{code}
\Crefname{tcb@cnt@pbox}{Code}{Code}
\crefname{assumption}{assumption}{assumption}
\Crefname{assumption}{Assumption}{Assumptions}

\newtcolorbox[auto counter]{pbox}[2][]{
  colback=white,
  title=Code~\thetcbcounter: #2,
  #1,fonttitle=\sffamily,
  fontupper=\sffamily,
  arc=2pt,
  colframe=bgcolor,
  coltitle=fgcolor,
  colbacktitle=bgcolor,
  toptitle=0.25cm,
  bottomtitle=0.125cm
}

\makeatletter
\newcommand\applefootnote[1]{%
  \begingroup
  \renewcommand\thefootnote{}%
  \renewcommand\@makefntext[1]{\noindent##1}%
  \footnote{#1}%
  \addtocounter{footnote}{-1}%
  \endgroup
}
\makeatother

\definecolor{cverbbg}{gray}{0.90}

%% file: math_commands.tex
\usepackage{amsmath,amsfonts,bm}

\def\eqref#1{equation~\ref{#1}}

\def\1{\bm{1}}

\def\rva{{\mathbf{a}}}

\def\rvh{{\mathbf{h}}}

\def\rvw{{\mathbf{w}}}
\def\rvx{{\mathbf{x}}}

\def\rmI{{\mathbf{I}}}

\def\rmU{{\mathbf{U}}}

\def\vzero{{\bm{0}}}

\def\vx{{\bm{x}}}

\DeclareMathAlphabet{\mathsfit}{\encodingdefault}{\sfdefault}{m}{sl}
\SetMathAlphabet{\mathsfit}{bold}{\encodingdefault}{\sfdefault}{bx}{n}

\def\gL{{\mathcal{L}}}

\def\gS{{\mathcal{S}}}

\newcommand{\E}{\mathbb{E}}
\newcommand{\Ls}{\mathcal{L}}
\newcommand{\R}{\mathbb{R}}

\newcommand{\reals}{\R}

\newcommand{\abs}[1]{\left \lvert #1 \right \rvert}

%% file: main.bbl
\begin{thebibliography}{45}
\providecommand{\natexlab}[1]{#1}
\providecommand{\url}[1]{\texttt{#1}}
\expandafter\ifx\csname urlstyle\endcsname\relax
  \providecommand{\doi}[1]{doi: #1}\else
  \providecommand{\doi}{doi: \begingroup \urlstyle{rm}\Url}\fi

\bibitem[Arora et~al.(2023)Arora, Eyuboglu, Timalsina, Johnson, Poli, Zou,
  Rudra, and R{\'e}]{arora2023zoology}
Simran Arora, Sabri Eyuboglu, Aman Timalsina, Isys Johnson, Michael Poli, James
  Zou, Atri Rudra, and Christopher R{\'e}.
\newblock Zoology: Measuring and improving recall in efficient language models.
\newblock \emph{arXiv preprint arXiv:2312.04927}, 2023.

\bibitem[Basant et~al.(2025)Basant, Khairnar, Paithankar, Khattar,
  Renduchintala, Malte, Bercovich, Hazare, Rico, Ficek,
  et~al.]{basant2025nvidia}
Aarti Basant, Abhijit Khairnar, Abhijit Paithankar, Abhinav Khattar, Adithya
  Renduchintala, Aditya Malte, Akhiad Bercovich, Akshay Hazare, Alejandra Rico,
  Aleksander Ficek, et~al.
\newblock Nvidia nemotron nano 2: An accurate and efficient hybrid
  mamba-transformer reasoning model.
\newblock \emph{arXiv preprint arXiv:2508.14444}, 2025.

\bibitem[Ben-Kish et~al.(2024)Ben-Kish, Zimerman, Abu-Hussein, Cohen,
  Globerson, Wolf, and Giryes]{ben2024decimamba}
Assaf Ben-Kish, Itamar Zimerman, Shady Abu-Hussein, Nadav Cohen, Amir
  Globerson, Lior Wolf, and Raja Giryes.
\newblock Decimamba: Exploring the length extrapolation potential of mamba.
\newblock \emph{arXiv preprint arXiv:2406.14528}, 2024.

\bibitem[Black et~al.(2022)Black, Biderman, Hallahan, Anthony, Gao, Golding,
  He, Leahy, McDonell, Phang, et~al.]{black2022gpt}
Sidney Black, Stella Biderman, Eric Hallahan, Quentin Anthony, Leo Gao,
  Laurence Golding, Horace He, Connor Leahy, Kyle McDonell, Jason Phang, et~al.
\newblock Gpt-neox-20b: An open-source autoregressive language model.
\newblock In \emph{Proceedings of BigScience Episode\# 5--Workshop on
  Challenges \& Perspectives in Creating Large Language Models}, pages 95--136,
  2022.

\bibitem[Blakeman et~al.(2025)Blakeman, Grattafiori, Basant, Gupta, Khattar,
  Renduchintala, Vavre, Shukla, Bercovich, Ficek, et~al.]{blakeman2025nvidia}
Aaron Blakeman, Aaron Grattafiori, Aarti Basant, Abhibha Gupta, Abhinav
  Khattar, Adi Renduchintala, Aditya Vavre, Akanksha Shukla, Akhiad Bercovich,
  Aleksander Ficek, et~al.
\newblock Nvidia nemotron 3: Efficient and open intelligence.
\newblock \emph{arXiv preprint arXiv:2512.20856}, 2025.

\bibitem[Chan et~al.(2022)Chan, Santoro, Lampinen, Wang, Singh, Richemond,
  McClelland, and Hill]{chan2022data}
Stephanie Chan, Adam Santoro, Andrew Lampinen, Jane Wang, Aaditya Singh, Pierre
  Richemond, James McClelland, and Felix Hill.
\newblock Data distributional properties drive emergent in-context learning in
  transformers.
\newblock \emph{Advances in neural information processing systems},
  35:\penalty0 18878--18891, 2022.

\bibitem[Dao and Gu(2024)]{dao2024transformers}
Tri Dao and Albert Gu.
\newblock Transformers are ssms: Generalized models and efficient algorithms
  through structured state space duality.
\newblock \emph{arXiv preprint arXiv:2405.21060}, 2024.

\bibitem[Dao et~al.(2022)Dao, Fu, Ermon, Rudra, and
  R{\'e}]{dao2022flashattention}
Tri Dao, Dan Fu, Stefano Ermon, Atri Rudra, and Christopher R{\'e}.
\newblock Flashattention: Fast and memory-efficient exact attention with
  io-awareness.
\newblock \emph{Advances in neural information processing systems},
  35:\penalty0 16344--16359, 2022.

\bibitem[Dong et~al.(2024)Dong, Li, Dai, Zheng, Ma, Li, Xia, Xu, Wu, Chang,
  et~al.]{dong2024survey}
Qingxiu Dong, Lei Li, Damai Dai, Ce~Zheng, Jingyuan Ma, Rui Li, Heming Xia,
  Jingjing Xu, Zhiyong Wu, Baobao Chang, et~al.
\newblock A survey on in-context learning.
\newblock In \emph{Proceedings of the 2024 conference on empirical methods in
  natural language processing}, pages 1107--1128, 2024.

\bibitem[Grazzi et~al.(2024)Grazzi, Siems, Schrodi, Brox, and
  Hutter]{grazzi2024mamba}
Riccardo Grazzi, Julien Siems, Simon Schrodi, Thomas Brox, and Frank Hutter.
\newblock Is mamba capable of in-context learning?
\newblock \emph{arXiv preprint arXiv:2402.03170}, 2024.

\bibitem[Gu and Dao(2023)]{gu2023mamba}
Albert Gu and Tri Dao.
\newblock Mamba: Linear-time sequence modeling with selective state spaces.
\newblock \emph{arXiv preprint arXiv:2312.00752}, 2023.

\bibitem[Gu et~al.(2021)Gu, Goel, and R{\'e}]{gu2021efficiently}
Albert Gu, Karan Goel, and Christopher R{\'e}.
\newblock Efficiently modeling long sequences with structured state spaces.
\newblock \emph{arXiv preprint arXiv:2111.00396}, 2021.

\bibitem[Huang et~al.(2025)Huang, Sarabia, Moudgil, Rodriguez, Zappella, and
  Danieli]{huang2025understanding}
Ningyuan Huang, Miguel Sarabia, Abhinav Moudgil, Pau Rodriguez, Luca Zappella,
  and Federico Danieli.
\newblock Understanding input selectivity in mamba: impact on approximation
  power, memorization, and associative recall capacity.
\newblock \emph{arXiv preprint arXiv:2506.11891}, 2025.

\bibitem[Jelassi et~al.(2024)Jelassi, Brandfonbrener, Kakade, and
  Malach]{jelassi2024repeat}
Samy Jelassi, David Brandfonbrener, Sham~M Kakade, and Eran Malach.
\newblock Repeat after me: Transformers are better than state space models at
  copying.
\newblock In \emph{International Conference on Machine Learning}, pages
  21502--21521. PMLR, 2024.

\bibitem[Jordan et~al.(2024)Jordan, Jin, Boza, Jiacheng, Cesista, Newhouse, and
  Bernstein]{jordan2024muon}
Keller Jordan, Yuchen Jin, Vlado Boza, You Jiacheng, Franz Cesista, Laker
  Newhouse, and Jeremy Bernstein.
\newblock Muon: An optimizer for hidden layers in neural networks, 2024.
\newblock \emph{URL https://kellerjordan. github. io/posts/muon}, 6\penalty0
  (3):\penalty0 4, 2024.

\bibitem[Katharopoulos et~al.(2020)Katharopoulos, Vyas, Pappas, and
  Fleuret]{katharopoulos2020transformers}
Angelos Katharopoulos, Apoorv Vyas, Nikolaos Pappas, and Fran{\c{c}}ois
  Fleuret.
\newblock Transformers are rnns: Fast autoregressive transformers with linear
  attention.
\newblock In \emph{International conference on machine learning}, pages
  5156--5165. PMLR, 2020.

\bibitem[Koren et~al.(2026)Koren, Ben-Kish, Giryes, Wolf, and
  Zimerman]{korenrecall}
Yuval Koren, Assaf Ben-Kish, Raja Giryes, Lior Wolf, and Itamar Zimerman.
\newblock On the recall scaling laws in mamba: A theoretical and mechanistic
  study via hashing.
\newblock \emph{arXiv preprint arXiv:2609.07681}, 2026.

\bibitem[Li et~al.(2025{\natexlab{a}})Li, Lu, Cui, Chen, and Wang]{li2025can}
Hongkang Li, Songtao Lu, Xiaodong Cui, Pin-Yu Chen, and Meng Wang.
\newblock Can mamba learn in context with outliers? a theoretical
  generalization analysis.
\newblock \emph{arXiv preprint arXiv:2510.00399}, 2025{\natexlab{a}}.

\bibitem[Li et~al.(2024)Li, Wei, Zhao, and Ma]{li2024can}
Yingcong Li, Xupeng Wei, Haonan Zhao, and Taigao Ma.
\newblock Can mamba in-context learn task mixtures?
\newblock In \emph{ICML 2024 Workshop on In-Context Learning}, 2024.

\bibitem[Li et~al.(2025{\natexlab{b}})Li, Tarzanagh, Rawat, Fazel, and
  Oymak]{li2025gating}
Yingcong Li, Davoud~Ataee Tarzanagh, Ankit~Singh Rawat, Maryam Fazel, and Samet
  Oymak.
\newblock Gating is weighting: Understanding gated linear attention through
  in-context learning.
\newblock \emph{arXiv preprint arXiv:2504.04308}, 2025{\natexlab{b}}.

\bibitem[Loshchilov and Hutter(2017)]{loshchilov2017decoupled}
Ilya Loshchilov and Frank Hutter.
\newblock Decoupled weight decay regularization.
\newblock \emph{arXiv preprint arXiv:1711.05101}, 2017.

\bibitem[Lu et~al.(2025{\natexlab{a}})Lu, Huang, Zeng, Wang, Chen, Langlais,
  and Cui]{lu2025mamba}
Peng Lu, Jerry Huang, Qiuhao Zeng, Xinyu Wang, Boxing Chen, Philippe Langlais,
  and Yufei Cui.
\newblock Mamba modulation: On the length generalization of mamba.
\newblock \emph{arXiv preprint arXiv:2509.19633}, 2025{\natexlab{a}}.

\bibitem[Lu et~al.(2025{\natexlab{b}})Lu, Letey, Zavatone-Veth, Maiti, and
  Pehlevan]{lu2025asymptotic_icl}
Yue~M Lu, Mary Letey, Jacob~A Zavatone-Veth, Anindita Maiti, and Cengiz
  Pehlevan.
\newblock Asymptotic theory of in-context learning by linear attention.
\newblock \emph{Proceedings of the National Academy of Sciences}, 122\penalty0
  (28):\penalty0 e2502599122, 2025{\natexlab{b}}.

\bibitem[Malach et~al.(2025)Malach, Saremi, Williamson, Bradley, Lotfi, Abbe,
  Susskind, and Littwin]{malach2025infinity}
Eran Malach, Omid Saremi, Sinead Williamson, Arwen Bradley, Aryo Lotfi,
  Emmanuel Abbe, Josh Susskind, and Etai Littwin.
\newblock To infinity and beyond: Tool-use unlocks length generalization in
  state space models.
\newblock \emph{arXiv preprint arXiv:2510.14826}, 2025.

\bibitem[Okpekpe and Orvieto(2025)]{okpekpe2025recalling}
Destiny Okpekpe and Antonio Orvieto.
\newblock When recalling in-context, transformers are not ssms.
\newblock \emph{arXiv preprint arXiv:2508.19029}, 2025.

\bibitem[Park et~al.(2024)Park, Park, Xiong, Lee, Cho, Oymak, Lee, and
  Papailiopoulos]{park2024can}
Jongho Park, Jaeseung Park, Zheyang Xiong, Nayoung Lee, Jaewoong Cho, Samet
  Oymak, Kangwook Lee, and Dimitris Papailiopoulos.
\newblock Can mamba learn how to learn? a comparative study on in-context
  learning tasks.
\newblock \emph{arXiv preprint arXiv:2402.04248}, 2024.

\bibitem[Patil et~al.(2024)Patil, Zhang, Wang, and Gonzalez]{patil2024gorilla}
Shishir~G Patil, Tianjun Zhang, Xin Wang, and Joseph~E Gonzalez.
\newblock Gorilla: Large language model connected with massive apis.
\newblock \emph{Advances in Neural Information Processing Systems},
  37:\penalty0 126544--126565, 2024.

\bibitem[Patil et~al.(2025)Patil, Mao, Cheng-Jie~Ji, Yan, Suresh, Stoica, and
  E.~Gonzalez]{patil2025bfcl}
Shishir~G. Patil, Huanzhi Mao, Charlie Cheng-Jie~Ji, Fanjia Yan, Vishnu Suresh,
  Ion Stoica, and Joseph E.~Gonzalez.
\newblock The berkeley function calling leaderboard (bfcl): From tool use to
  agentic evaluation of large language models.
\newblock In \emph{Forty-second International Conference on Machine Learning},
  2025.

\bibitem[Qin et~al.(2023)Qin, Liang, Ye, Zhu, Yan, Lu, Lin, Cong, Tang, Qian,
  et~al.]{qin2023toolllm}
Yujia Qin, Shihao Liang, Yining Ye, Kunlun Zhu, Lan Yan, Yaxi Lu, Yankai Lin,
  Xin Cong, Xiangru Tang, Bill Qian, et~al.
\newblock Toolllm: Facilitating large language models to master 16000+
  real-world apis.
\newblock In \emph{The twelfth international conference on learning
  representations}, 2023.

\bibitem[Qwen(2026)]{qwen35blog}
Alibaba Qwen.
\newblock Qwen3.5: Accelerating productivity with native multimodal agents,
  February 2026.
\newblock URL \url{https://qwen.ai/blog?id=qwen3.5}.

\bibitem[Reddy(2023)]{reddy2023mechanistic}
Gautam Reddy.
\newblock The mechanistic basis of data dependence and abrupt learning in an
  in-context classification task.
\newblock \emph{arXiv preprint arXiv:2312.03002}, 2023.

\bibitem[Ren et~al.(2025)Ren, Li, and Liu]{ren2025exploring}
Ruifeng Ren, Zhicong Li, and Yong Liu.
\newblock Exploring the limitations of mamba in copy and cot reasoning.
\newblock In \emph{Proceedings of the 2025 Conference on Empirical Methods in
  Natural Language Processing}, pages 12550--12574, 2025.

\bibitem[Robertson and Zaragoza(2009)]{robertson2009bm25}
Stephen Robertson and Hugo Zaragoza.
\newblock \emph{The probabilistic relevance framework: BM25 and beyond},
  volume~4.
\newblock Now Publishers Inc, 2009.

\bibitem[Ruiz and Gu(2025)]{ruiz2025understanding}
Ricardo~Buitrago Ruiz and Albert Gu.
\newblock Understanding and improving length generalization in recurrent
  models.
\newblock \emph{arXiv preprint arXiv:2507.02782}, 2025.

\bibitem[Su et~al.(2025)Su, Kong, Lin, Jennings, Norick, Kliegl, Patwary,
  Shoeybi, and Catanzaro]{su2025nemotron}
Dan Su, Kezhi Kong, Ying Lin, Joseph Jennings, Brandon Norick, Markus Kliegl,
  Mostofa Patwary, Mohammad Shoeybi, and Bryan Catanzaro.
\newblock Nemotron-cc: Transforming common crawl into a refined long-horizon
  pretraining dataset.
\newblock In \emph{Proceedings of the 63rd Annual Meeting of the Association
  for Computational Linguistics (Volume 1: Long Papers)}, pages 2459--2475,
  2025.

\bibitem[Su et~al.(2024)Su, Ahmed, Lu, Pan, Bo, and Liu]{su2024roformer}
Jianlin Su, Murtadha Ahmed, Yu~Lu, Shengfeng Pan, Wen Bo, and Yunfeng Liu.
\newblock Roformer: Enhanced transformer with rotary position embedding.
\newblock \emph{Neurocomputing}, 568:\penalty0 127063, 2024.

\bibitem[Trockman et~al.(2024)Trockman, Harutyunyan, Kolter, Kumar, and
  Bhojanapalli]{trockman2024mimetic}
Asher Trockman, Hrayr Harutyunyan, J~Zico Kolter, Sanjiv Kumar, and Srinadh
  Bhojanapalli.
\newblock Mimetic initialization helps state space models learn to recall.
\newblock \emph{arXiv preprint arXiv:2410.11135}, 2024.

\bibitem[Waleffe et~al.(2024)Waleffe, Byeon, Riach, Norick, Korthikanti, Dao,
  Gu, Hatamizadeh, Singh, Narayanan, et~al.]{waleffe2024empirical}
Roger Waleffe, Wonmin Byeon, Duncan Riach, Brandon Norick, Vijay Korthikanti,
  Tri Dao, Albert Gu, Ali Hatamizadeh, Sudhakar Singh, Deepak Narayanan, et~al.
\newblock An empirical study of mamba-based language models.
\newblock \emph{arXiv preprint arXiv:2406.07887}, 2024.

\bibitem[Wang et~al.(2024)Wang, Cai, Wang, Zhu, Srivastava, Wang, and
  Li]{wang2024understanding}
Peihao Wang, Ruisi Cai, Yuehao Wang, Jiajun Zhu, Pragya Srivastava, Zhangyang
  Wang, and Pan Li.
\newblock Understanding and mitigating bottlenecks of state space models
  through the lens of recency and over-smoothing.
\newblock \emph{arXiv preprint arXiv:2501.00658}, 2024.

\bibitem[Wang et~al.(2025)Wang, Chen, Yan, Lu, and Sun]{wang2025memmamba}
Youjin Wang, Yangjingyi Chen, Jiahao Yan, Jiaxuan Lu, and Xiao Sun.
\newblock Memmamba: Rethinking memory patterns in state space model.
\newblock \emph{arXiv preprint arXiv:2510.03279}, 2025.

\bibitem[Yang et~al.(2023)Yang, Wang, Shen, Panda, and Kim]{yang2023gated}
Songlin Yang, Bailin Wang, Yikang Shen, Rameswar Panda, and Yoon Kim.
\newblock Gated linear attention transformers with hardware-efficient training.
\newblock \emph{arXiv preprint arXiv:2312.06635}, 2023.

\bibitem[Yang et~al.(2024)Yang, Kautz, and Hatamizadeh]{yang2024gated}
Songlin Yang, Jan Kautz, and Ali Hatamizadeh.
\newblock Gated delta networks: Improving mamba2 with delta rule.
\newblock \emph{arXiv preprint arXiv:2412.06464}, 2024.

\bibitem[You et~al.(2025)You, Tang, Li, Yao, and Zhang]{you2025revealing}
Wangjie You, Zecheng Tang, Juntao Li, Lili Yao, and Min Zhang.
\newblock Revealing and mitigating the local pattern shortcuts of mamba.
\newblock In \emph{Findings of the Association for Computational Linguistics:
  ACL 2025}, pages 12156--12178, 2025.

\bibitem[Zhan et~al.(2025)Zhan, Zhao, Zhu, and Tang]{zhan2025overcoming}
Zhihao Zhan, Jianan Zhao, Zhaocheng Zhu, and Jian Tang.
\newblock Overcoming long-context limitations of state-space models via
  context-dependent sparse attention.
\newblock \emph{arXiv preprint arXiv:2507.00449}, 2025.

\bibitem[Zhang et~al.(2024)Zhang, Frei, and Bartlett]{zhang2024transformer_icl}
Ruiqi Zhang, Spencer Frei, and Peter~L Bartlett.
\newblock Trained transformers learn linear models in-context.
\newblock \emph{Journal of Machine Learning Research}, 25\penalty0
  (49):\penalty0 1--55, 2024.

\end{thebibliography}
